\documentclass[final]{nesy2026} 

\usepackage{longtable}

\usepackage{booktabs}

\usepackage{multirow}
\usepackage{amsmath,amssymb}
\usepackage{graphicx}
\usepackage{float}   
\usepackage{wrapfig}
\usepackage{url}
\usepackage{tikz}
\usetikzlibrary{arrows.meta,positioning}
\usepackage[load-configurations=version-1]{siunitx}

\definecolor{oiBlue}{HTML}{0072B2}
\definecolor{oiGreen}{HTML}{009E73}
\definecolor{oiOrange}{HTML}{E69F00}
\definecolor{oiVerm}{HTML}{D55E00}
\definecolor{oiPurple}{HTML}{CC79A7}

\usepackage[load-configurations=version-1]{siunitx} 

\theorembodyfont{\upshape}
\theoremheaderfont{\scshape}
\theorempostheader{:}
\theoremsep{\newline}

\title[FSG-LTN-GAN]{Generate in the Chart, Not on the Boundary: Function-Symbol Grounding for Hard Constraints in~LTN-GANs}

\clearauthor{\Name{Nijesh Upreti} \Email{n.upreti@ed.ac.uk}\\
        \Name{Vaishak Belle} \Email{vbelle@ed.ac.uk}\\
        \addr The University of Edinburgh, 10 Crichton Street, Edinburgh, EH8 9AB, UK}

\begin{document}
\maketitle

\begin{abstract}
Logic Tensor Network-Enhanced Generative Adversarial Networks (LTN-GANs) inject background knowledge by grounding each logical axiom as a predicate and training the generator to raise its satisfaction, a fuzzy truth value in $[0,1]$. Previous LTN-GAN work grounded every constraint this way, at the predicate level, and improved constraint satisfaction. A predicate, however, only scores a sample, so it cannot embed hard structural constraints, rules such as orderings, positivity, and definitional identities that must hold in every generated sample. In this work, we investigate grounding each axiom as a function symbol inside the LTN framework. We compare against the state-of-the-art alternative, a constraint layer that clamps each violating sample onto the feasible boundary and so produces outputs that are always valid. Our investigation shows that a valid sample is not always a realistic one. An inequality is not merely satisfied or violated. It holds by a margin, and a faithful generator should also reproduce the margin's real distribution. We find that the resolution ratio $R$, the data's scale over the margin's spread, is a diagnostic, computable before training, of which constraints a chosen grounding can learn. When $R$ is large, the predicate receives no learning signal, the clamp pushes every sample onto the boundary, and the margin distribution is lost while every standard metric still looks fine. A function symbol avoids both failures, computing the constrained variable rather than scoring it. Together the function symbols form a chart, a coordinate system inside the feasible region, where every sample is valid by construction and the margin is learned like any other quantity. This recovers the margin distribution on four high-resolution datasets, with Kolmogorov--Smirnov distance up to $25\times$ smaller than the constraint layer's, and a hybrid of charting and clamping matches or exceeds the constraint layer on its own benchmark. 
\end{abstract}

\section{Introduction}

Deep generative models are a standard tool for producing synthetic tabular and
scientific data, used to augment scarce records, share sensitive data, and
propose candidate designs. Adversarial and variational generators such as CTGAN and TVAE
\citep{xu2019modeling}, score-based and diffusion models
\citep{kim2023stasy,kotelnikov2023tabddpm}, and relational-structure models
\citep{liu2022goggle} learn increasingly accurate approximations of the distributions of data
such as trip records, flight records, and molecular property profiles.

However, approximating the data distribution well is not enough, because such data obey known
rules that an approximation can still violate. For example, a synthetic trip record must have
its drop-off after its pick-up, a flight must report its departure as schedule plus delay, and
a molecule's internal energy must rise from $0$\,K ($U_0$) to room temperature ($U$) and
remain below its enthalpy ($H$), so that $U_0<U<H$. \emph{Constrained generation}
therefore asks for samples that are both realistic and provably valid under such rules, which in
tabular and scientific data are typically linear orderings between properties, positivity
requirements, and definitional identities. Two main families of
methods inject this knowledge into a generator. The first adds a penalty to
the training loss whenever a sample violates a constraint. Such a constraint is called
\emph{soft}, since the generator can trade it against the rest of the loss, and satisfaction
is encouraged but never guaranteed. The Logic Tensor Network-Enhanced
Generative Adversarial Network (LTN-GAN) \citep{upreti2026ltngan} \emph{grounds} each
constraint as a \emph{predicate}, a differentiable function that scores how well a sample
satisfies it, and trains the generator to raise this score alongside its adversarial
objective \citep{badreddine2022ltn}. The second family builds the constraints into the model,
so that a violating sample cannot be produced. Such a constraint is called \emph{hard},
since no output can break it. The differentiable \emph{constraint
layer} of \citet{stoian2024cdgm}, which turns a generator into a Constrained Deep Generative
Model (C-DGM) and is the state of the art for tabular data, moves each violating sample
into the feasible region by clamping and guarantees $100\%$ validity for any conjunction
of linear inequalities.

Our starting observation is that validity alone does not make constrained data realistic. A
valid sample satisfies each inequality by some amount, and validity ignores how
these amounts are distributed. In a trip record, the time from pick-up to drop-off is the
trip's duration. A generator can place every drop-off after its pick-up yet produce durations
unlike those of real trips. The energy step $U-U_0$ can fail the same way. We call each such amount
(drop-off minus pick-up, or $U-U_0$) the constraint's \emph{margin}, and we call its distribution over the real data the
\emph{distribution of the constraint margin}, which a faithful generator must reproduce.
Whether a generator can reproduce this distribution depends on the \emph{resolution ratio} $R$, the
data's scale over the margin's spread, computed before training. When
$R$ is small, the margin is visible at the data's scale, so an unconstrained generator places it,
the satisfaction score has usable gradient, and a clamp rarely fires. When $R$ is large,
the margin is invisible at the data's scale. For thermochemical energies the
step $U-U_0$ is five orders of magnitude smaller than the energies themselves. In this
regime the discriminator can no longer resolve the margin, the satisfaction score has
vanishing gradient,
and the generator violates the constraint on a constant fraction of samples. Trained in the
loop, the constraint layer moves nearly every sample onto the boundary, collapsing the margin
distribution to a point mass while validity reads $100\%$ and every per-feature statistic
still matches the real data (margin Kolmogorov--Smirnov (KS) distance near $1$; Figure~\ref{fig:margins}). High-$R$ constraints are common in scientific data, yet no
standard metric detects the collapse.

We show that the LTN framework already contains the mechanism to avoid this collapse. A constraint
can also be grounded through a \emph{function symbol}, a term that
computes the constrained variable directly. For an ordering $b>a$, the generator emits a free value for $a$ and
assembles $b$ from $a$ by adding a positive, smoothly parameterised increment. Every sample
satisfies the constraint \emph{by construction}, and the
margin becomes a coordinate that the
discriminator can resolve and shape at unit scale. The function symbols form a
\emph{chart} of the feasible region, a coordinate system in which the generator
produces samples inside the region rather than on its boundary. The chart's per-variable
admissible intervals are exactly those the constraint layer computes
by Fourier-Motzkin reduction, and the constraint layer becomes the
special case that always clamps. We call the result the
\emph{FSG-LTN-GAN}, for \emph{function-symbol grounding} (FSG) in LTN-GANs.
Since charting helps on high-$R$ continuous constraints while low-$R$ and discrete margins are
better clamped, we use a short pre-run to decide per constrained variable, yielding a \textbf{hybrid}
model. We show that $R$ acts as a condition
number for the constrained density-estimation problem (Section~\ref{sec:method}).

\paragraph{Contributions.} (i) We show that hard-constrained generation can distort the
\emph{distribution of the constraint margin} while passing every standard metric, and that the
resolution ratio $R$ acts as a condition number that predicts this failure before training.
(ii) We develop \emph{function-symbol grounding}, a change of coordinates that generates inside
the feasible region with exact validity and subsumes the constraint layer as its
clamp-everything special case, and extend it to a per-constraint \emph{hybrid}. (iii) We
demonstrate that function-symbol grounding cuts the margin KS by up to $25\times$ at equal
validity on four real high-resolution datasets, that the hybrid matches or exceeds the
constraint layer on the benchmark of \citet{stoian2024cdgm}, that the $R$-based prediction holds out
of sample (nycflights13), and that the method transfers unchanged to
CTGAN and TVAE backbones.

\section{Background: Logic Tensor Networks and LTN-GANs}\label{sec:background}

\paragraph{Real Logic and grounding.} Logic Tensor Networks (LTN) interpret a first-order
language, Real Logic, in real-valued tensors \citep{badreddine2022ltn}. A \emph{grounding}
$\mathcal{G}$ assigns meaning to symbols, and every term becomes a tensor. Each
\emph{predicate} $P$ becomes a map $\mathcal{G}(P)$ into the truth interval $[0,1]$ that
\emph{scores} its arguments. Each \emph{function symbol} $f$ of arity $k$ becomes a real map
$\mathcal{G}(f):\mathbb{R}^{Dk}\!\to\!\mathbb{R}^D$,
fixed or learnable, that \emph{computes} a term. A predicate acts on learning only through the truth values it contributes,
while a function symbol shapes the object being scored. Logical connectives become fuzzy
operators (a t-norm for $\wedge$, its dual for $\vee$, a fuzzy implication), and quantifiers
become aggregations ($\forall$ as a generalized mean over errors). The truth value of a
closed formula under $\mathcal{G}$ is its \emph{satisfaction} $\mathrm{Sat}\in[0,1]$, with
$1$ fully true. An \emph{axiom} is a closed formula asserted to hold of the domain, a
knowledge base $\mathrm{KB}$ is a finite set of axioms, and learning maximizes their
aggregated satisfaction $\mathrm{Sat}(\mathrm{KB})$.

\paragraph{The LTN-GAN objective.} A GAN couples a generator $G_\theta:\mathcal{Z}\!\to\!
\mathbb{R}^D$, mapping latent noise $\zeta\sim p_\zeta$ (a
standard Gaussian) to a sample, and a discriminator
$D_\psi:\mathbb{R}^D\!\to\![0,1]$, trained to score real samples near $1$
and generated ones near $0$, while $G_\theta$ is trained to fool it. An
LTN-GAN treats the generated sample as the grounding of the constrained variables and trains
the generator to also satisfy a knowledge base, using the objective
$\mathcal{L}_G=\mathcal{L}_{\mathrm{adv}}(G_\theta,D_\psi)+\lambda\,(1-\mathrm{Sat}(\mathrm{KB}))$,
whose axioms encode the constraints.

\paragraph{Predicate grounding of a constraint.} The standard generator-side LTN-GAN (G-LTN-GAN) \citep{upreti2026ltngan} grounds an ordering axiom
$b>a$ as a predicate, $\mathcal{G}(P_{>})(a,b)=\sigma\!\bigl((b-a)/s\bigr)$ with $\sigma$ the
logistic, at a band $s$ (a width hyperparameter),
and the term $\lambda(1-\mathrm{Sat})$ in the objective pushes samples toward
$b>a$. This grounding is \emph{soft}. It encourages, but does not guarantee, satisfaction, and
its gradient is informative only where the predicate is unsaturated. Section~\ref{sec:method}
shows this band covers a fraction $\Theta(1/R)$ of samples, so the usable gradient vanishes as $R$ grows (RQ4). Our method keeps the
LTN-GAN objective but regrounds each structural axiom through a function symbol. The
constrained variable becomes a term that $\mathcal{G}(f)$ computes rather than a free output
that a predicate scores, so $\mathrm{Sat}=1$ holds by construction. \emph{Soft} thus refers
to the satisfaction signal, since function-symbol grounding also uses smooth links yet
satisfies the axiom for every output.

\section{Problem Statement}\label{sec:problem}

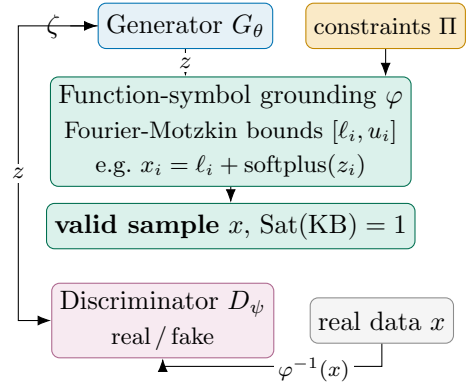
\begin{wrapfigure}[21]{r}{0.44\linewidth}
\makeatletter\long\def\@makecaption#1#2{\vskip 4pt #1: #2\par}\makeatother
\vspace{-0.9\baselineskip}
\centering
\begin{tikzpicture}[
  >={Latex[length=1.8mm]}, font=\small,
  box/.style={draw, rounded corners=3pt, align=center, inner sep=3pt, minimum height=6mm},
  gen/.style={box, fill=oiBlue!10, draw=oiBlue},
  grnd/.style={box, fill=oiGreen!14, draw=oiGreen!70!black},
  disc/.style={box, fill=oiPurple!16, draw=oiPurple!85!black},
  data/.style={box, fill=black!4, draw=black!45},
  cons/.style={box, fill=oiOrange!22, draw=oiOrange!85!black, font=\footnotesize},
  lab/.style={font=\footnotesize, inner sep=1.5pt, fill=white}]
  \node[gen] (gen) at (-0.6,0) {Generator $G_\theta$};
  \node (z) at (-2.35,0) {$\zeta$};
  \node[cons] (pi) at (2.05,0) {constraints $\Pi$};
  \node[grnd] (phi) at (0,-1.4) {Function-symbol grounding $\varphi$\\[1pt]
    \footnotesize Fourier-Motzkin bounds $[\ell_i,u_i]$\\[1pt]
    \footnotesize e.g.\ $x_i=\ell_i+\mathrm{softplus}(z_i)$};
  \node[grnd] (x) at (0,-2.65) {\textbf{valid sample} $x$,\ $\mathrm{Sat(KB)}=1$};
  \node[disc] (disc) at (-0.9,-3.9) {Discriminator $D_\psi$\\[1pt]\footnotesize real\,/\,fake};
  \node[data] (real) at (2.0,-3.9) {real data $x$};
  \draw[->] (z)--(gen);
  \draw[->] (gen.south)--(gen.south|-phi.north) node[lab,midway]{$z$};
  \draw[->] (pi.south)--(pi.south|-phi.north);
  \draw[->] (phi)--(x);
  \draw[->] (gen.west)--++(-1.05,0)|-(disc.west) node[lab,pos=0.25]{$z$};
  \draw[->] (real.south) --++(0,-0.3) -| (disc.south);
  \node[lab] at (1.1,-4.55) {\scriptsize $\varphi^{-1}(x)$};
\end{tikzpicture}
\caption{Overview of the FSG-LTN-GAN. The grounding map $\varphi$ assembles valid samples from
the generator's free coordinates $z$ within the Fourier-Motzkin bounds of $\Pi$, and the
discriminator operates in the chart coordinates. The hybrid instead clamps low-$R$ and
discrete variables in the loop (Appendix~\ref{app:alg}).}
\label{fig:method}
\end{wrapfigure}
Let $p_X$ be an unknown distribution over $X\in\mathbb{R}^{D}$ and $\mathcal{D}$ a dataset of $N$
i.i.d.\ samples. A \emph{sample} is a vector $x=(x_1,\dots,x_D)$ whose scalar components
$x_k\in\mathbb{R}$ are its \emph{features}. A generative model with parameters $\theta$ (for us
the generator $G_\theta$) induces the distribution $p_\theta$ of
its output $G_\theta(\zeta)$, $\zeta\sim p_\zeta$, and learning chooses $\theta$ so that
$p_\theta\approx p_X$. The background
knowledge is a finite set $\Pi$ of linear-inequality axioms over the features
$\{x_1,\dots,x_D\}$, each of the form $\sum_k w_k x_k + b \trianglerighteq 0$ with
$\trianglerighteq\in\{\ge,>\}$, real coefficients $w_k$, and offset $b$, following the
formulation of \citet{stoian2024cdgm}. A sample $\tilde x$ \emph{satisfies}
$\phi\in\Pi$ if $\sum_k w_k\tilde x_k+b\trianglerighteq 0$. A generator is \emph{compliant}
(valid) if all its samples satisfy all of $\Pi$. Orderings ($x_i>x_j$), positivity ($x_i>0$),
and definitional identities ($x_i=\sum_j w_jx_j$, encoded as two inequalities) are the
\emph{structural} fragment we study.

\paragraph{The margin and its distribution.} For $\phi:\sum_k w_kx_k+b\ge 0$, its \emph{margin}
on a sample is $m_\phi(x)=\sum_k w_kx_k+b$, with $m_\phi\ge 0$ exactly when $\phi$ holds. The right target is the \emph{distribution of the constraint margin}, the
conditional distribution of the margin under the data, $p_X(m_\phi\mid m_\phi\ge 0)$. We
quantify this by the Kolmogorov--Smirnov (KS) distance between the generated and real constraint
margins on a fixed reference sample of the real data.

\paragraph{The resolution ratio.} The \emph{spread}
$\sigma_m=\mathrm{std}_{\mathcal{D}}(m_\phi)$ is the standard deviation of $\phi$'s margin
over the dataset. The \emph{scale} $\sigma_s=\max_{k:\,w_k\neq0}\mathrm{std}_{\mathcal{D}}(x_k)$
is the largest standard deviation among the features $\phi$ relates: for $\phi:U-U_0\ge0$ it
is $\mathrm{std}(U)$. Both are computed on the raw, unstandardized data,
before training. The \emph{resolution ratio} of $\phi$ is $R_\phi=\sigma_s/\sigma_m$. When $R_\phi$ is
large the margin is a tiny difference of large near-equal quantities. In standardized
coordinates it occupies a band of relative width about $1/R_\phi$, sub-resolution to a Lipschitz
discriminator, so a \emph{free} generator (one trained with no constraint mechanism) that has matched the marginals still violates $\phi$ on a
constant fraction of samples (RQ4). Section~\ref{sec:method} casts $R_\phi$ as a scaling
condition number for recovering the distribution of the constraint margin.

\paragraph{The constraint layer (clamping).} C-DGM \citep{stoian2024cdgm} appends a differentiable
\emph{constraint layer} (CL): given a
variable ordering, it computes for each variable an admissible interval $[\ell_i,u_i]$ (piecewise-linear in
the already-set variables, by Fourier-Motzkin reduction) and \emph{clamps} the generated value
into it, $\mathrm{CL}(\tilde x)_i=\min(\max(\tilde x_i,\ell_i),u_i)$, guaranteeing validity.
A clamp moves a violating sample to the nearest boundary, where the
margin is near zero. Section~\ref{sec:theory} shows this is exactly where the
distribution of the constraint margin is lost, at a rate governed by $R$.

\section{Function-Symbol Grounding as a Change of Coordinates}\label{sec:method}

\emph{Change of coordinates} and \emph{chart} carry their differential-geometric meaning
throughout. Inside the affine subspace its identities define, the feasible set
$\mathcal{M}=\{x:\bigwedge_{\phi}\phi(x)\}$ of a satisfiable linear system is a relatively
open convex polytope, diffeomorphic to $\mathbb{R}^{d}$ and covered by a single global chart.
The map $\varphi$ below is such a chart, with one coordinate per free variable.

We ground each structural axiom through a
function symbol (Figure~\ref{fig:method}). The generator emits free terms
$z\in\mathbb{R}^{d}$ (one per variable not fixed by an identity), and a grounded map $\varphi$ assembles the
constrained sample by processing the variables in the Fourier-Motzkin order. With
$\Pi_i^{+}$ ($\Pi_i^{-}$) the reduced constraints in which $x_i$ has positive (negative)
coefficient \citep{stoian2024cdgm}, the admissible interval of $x_i$ is piecewise-linear in
the already-assembled $x_{<i}$,
\begin{equation}\label{eq:bounds}
\ell_i(x_{<i})=\max_{\phi\in\Pi_i^{+}}\varepsilon_i^{\phi}(x_{<i}),\quad
u_i(x_{<i})=\min_{\phi\in\Pi_i^{-}}\varepsilon_i^{\phi}(x_{<i}),\quad
\varepsilon_i^{\phi}=-\Bigl(\textstyle\sum_{k<i}w_{k}x_{k}+b\Bigr)/w_{i},
\end{equation}
the bounds the constraint layer clamps into (Section~\ref{sec:problem}). Per variable,
$\varphi$ applies the grounding that fits the interval:
\begin{align*}
\text{one-sided: } x_i&=\ell_i+\mathrm{softplus}(z_i)\ \text{ or }\ u_i-\mathrm{softplus}(z_i),
&\text{box: } x_i&=\ell_i+(u_i-\ell_i)\,\sigma(z_i),\\
\text{free: } x_i&=z_i,
&\text{identity: } x_i&=\textstyle\sum_{j<i} w_{ij}\,x_j+w_{i0}.
\end{align*}
(For Alchemy's $U_0<U<H$: $U_0=z_1$, $U=U_0+\mathrm{softplus}(z_2)$,
$H=U+\mathrm{softplus}(z_3)$.)
Each bounded variable is a smooth, monotone function of a unit-scale coordinate $z_i$ that stays
within its admissible interval, and each identity derives its dependent variable, so the axiom
holds by construction ($\mathrm{Sat}(\mathrm{KB})=1$, hence the logical term in
the objective vanishes for these axioms). The decoder
$\varphi:z\mapsto x$ is the \emph{chart} of $\mathcal{M}$ defined above, turning unconstrained
coordinates $z$ into feasible samples $x$. It inverts in closed form (for a one-sided variable,
$z_i=\mathrm{softplus}^{-1}(x_i-\ell_i)$), which gives the encoder $\varphi^{-1}$ applied to
real data. The discriminator operates on the chart coordinates
$z$, where every margin is unit scale, receiving $z$ for generated samples and
$\varphi^{-1}(x)$ for real ones, so a standard discriminator can learn the distribution of the
constraint margin directly. For heavy-tailed margins we
ground through $\exp$ rather than softplus (a multiplicative increment), set by a fixed
dynamic-range rule (Appendix~\ref{app:hparams}). Monotone softplus links enforcing order and positivity go back to
\citet{dugas2009incorporating}. Function-symbol grounding deploys them as the groundings of
structural axioms inside adversarial training.

\begin{proposition}[Validity by construction]\label{thm:valid}
For any satisfiable finite set $\Pi$ of linear inequalities and any generator output $z$, the
assembled sample $\varphi(z)$ satisfies $\Pi$.
\end{proposition}
\noindent The proof (Appendix~\ref{app:proof}) is the soundness of Fourier-Motzkin elimination.
In the elimination order each variable's admissible interval is non-empty given the finalized
earlier variables, and softplus and $\sigma$ map $\mathbb{R}$ into its interior, so every axiom
is met. Validity is exact, as for the constraint layer. The two differ only in \emph{where
in the interval} the sample lands, and that is what the distribution of the constraint margin captures.

\subsection{Why the chart can preserve the distribution of the constraint margin and the clamp cannot}\label{sec:theory}
For $\phi:b>a$, the clamp sends every violator to $b=a$, while
function-symbol grounding emits the margin as a unit-scale coordinate,
$m_\phi=\mathrm{softplus}(z)$. Remark~\ref{rem:cond} casts $R$ as a condition number, and
Proposition~\ref{thm:collapse} quantifies both mechanisms.

\begin{remark}[$R$ as a condition number]\label{rem:cond}
Recovering the distribution of the constraint margin of $\phi$ requires resolving a margin of scale
$\sigma_m=\mathrm{std}(m_\phi)$ inside data of scale $\sigma_s=R\,\sigma_m$. A Lipschitz
discriminator on the ambient coordinates must resolve a relative magnitude $1/R$. We describe this
as a condition number $\Theta(R)$ in the scaling sense, versus $\Theta(1)$ in the chart,
where the margin is standardized.
\end{remark}

\begin{corollary}[Definitional identities are measure-zero]\label{cor:identity}
Let $\mathcal{M}_{=}=\{x\in\mathbb{R}^{D}:c(x)=0\}$ be the zero set of the identities, with
$c:\mathbb{R}^{D}\to\mathbb{R}^{p}$ a $C^{1}$ map whose Jacobian has full rank $p$ on
$\mathcal{M}_{=}$. Then $\mathcal{M}_{=}$ is Lebesgue-null, so any generated distribution $\nu$ that is
absolutely continuous has $\Pr_{x\sim\nu}[c(x)=0]=0$, and no satisfaction loss can raise exact
satisfaction above probability $0$. Function-symbol grounding derives the dependent variables
from their parents, so $c\equiv0$ holds with probability $1$.
\end{corollary}

\begin{corollary}[Predicate-grounded ordering: $\Theta(1/R)$ gradient]\label{cor:ordering}
Ground an ordering $\phi:b>a$ as the predicate $P_{s}=\sigma\!\big((b-a)/s\big)$ with band
$s=\Theta(\sigma_{m})$. If the generator has matched the marginals of $a$ and $b$ but not their
dependence (margin correlation bounded away from $1$), its margin has spread
$\Theta(R\,\sigma_{m})$ and density $\Theta(1/\sigma_{s})$ near $b=a$ (no anomalous
concentration). The predicate gradient $P_{s}(1-P_{s})/s$ is $\Theta(1/s)$ on the band
$|b-a|\lesssim s$ and exponentially small outside it, so the fraction of generated samples with
usable gradient is $\Theta(1/R)$, vanishing as $R\to\infty$.
\end{corollary}

\begin{proposition}[Margin laws]\label{thm:collapse}
Let the real margin of $\phi:b>a$ have CDF $F$ supported on $(0,\infty)$, and clamp a free
generator with violation rate $v$ post hoc, with offset $\varepsilon$ below the real
support. (i) The clamped margin distribution is $v\,\delta_{\varepsilon}+(1-v)\,\nu^{+}$, with
$\nu^{+}$ the generator's satisfying-margin distribution. Under the hypotheses of
Corollary~\ref{cor:ordering}, its KS distance to $F$ is at least
$\max(v,1-v)-\Theta(1/R)\ge\tfrac12-\Theta(1/R)$, tending to $1$ as $v\to1$ (observed
in-loop; Section~\ref{sec:exp}). (ii) The chart $\varphi$ is a bijection from $\mathbb{R}^{d}$ onto the relative
interior of $\mathcal{M}$ with closed-form inverse, both smooth off the measure-zero set
where the active Fourier-Motzkin bound switches,
places no probability on the boundary, and realizes every margin distribution on $(0,\infty)$
exactly (proofs in Appendix~\ref{app:theory}).
\end{proposition}

\subsection{Which axioms to ground through function symbols: a hybrid}\label{sec:hybrid}
Function-symbol grounding is most effective where clamping is most costly, on high-$R$ constraints the
generator cannot meet on its own. But a continuous chart cannot represent a margin with a
discrete point mass (an integer count, a two-valued category). The softplus or $\sigma$ increment
smears it, whereas the in-loop clamp, which the generator learns to anticipate, matches it
empirically (RQ3). We therefore decide per bounded
variable, over the conjunction of the reduced constraints that bound it. A short pre-run of a
free generator measures the fraction $s_i$ of its samples that already lie in variable $x_i$'s
admissible interval, and the largest single-value frequency of $x_i$'s binding margin in
$\mathcal{D}$ measures the discreteness $d_i$. We chart $x_i$ if $s_i<0.9$ and $d_i\le0.2$
(violated and continuous). Otherwise we clamp it in the loop, with the discriminator seeing the
clamped value so the generator adapts to it. The resulting generator charts the high-$R$
continuous variables and clamps the rest. Validity remains exact
(Proposition~\ref{thm:valid}) since every bounded variable is either charted or clamped. Algorithm~\ref{alg:hybrid} (Appendix~\ref{app:alg})
summarizes the procedure.

\section{Experimental Analysis}\label{sec:exp}

We answer four questions. \textbf{RQ1:} does guaranteeing validity guarantee a realistic
constrained quantity? \textbf{RQ2:} does the effect persist across generator architectures?
\textbf{RQ3:} does the hybrid generalize to the constraint-layer benchmark of \citet{stoian2024cdgm}? \textbf{RQ4:}
can predicate grounding (the standard LTN-GAN satisfaction loss) reach high-$R$ constraints
instead?

\textbf{Datasets.} We use four real high-resolution datasets, each carrying an ordering whose
margin is a small difference of large quantities: \textbf{Alchemy} \citep{chen2019alchemy}
($U_0<U<H$), \textbf{tmQM} \citep{balcells2020tmqm}, \textbf{Transition1x}
\citep{schreiner2022transition1x} (reaction barriers), and \textbf{Taxi} \citep{nyctlc} (trip
duration). Constraints and per-margin $R$ ($3.6$ to $7{\times}10^6$) are in Table~\ref{tab:dataspec},
Appendix~\ref{app:data}. For RQ3 we use the six-dataset benchmark of
\citet{stoian2024cdgm}. \textbf{Baselines.} The \emph{constraint layer}
(C-DGM) on the same MLP-GAN as FSG-LTN-GAN; \emph{CTGAN} and \emph{TVAE} \citep{xu2019modeling}
with and without it; and post-hoc \emph{projection}. \textbf{Metrics.} Validity, the \emph{margin
KS} against a fixed reference sample of the real data (Appendix~\ref{app:fair}), and per-property moment error, as means $\pm$ std over $n{=}10$
seeds with paired Wilcoxon $p$-values. Full tables and the constraint-layer corrections are
in Appendices~\ref{app:fair},~\ref{app:fullrq},~\ref{app:rq3full}, and~\ref{app:extra}.

\subsection{RQ1: validity does not imply a realistic constrained quantity}\label{sec:rq1}

\begin{table}[t]\centering
\caption{\textbf{Validity does not capture the distribution of the constraint margin (RQ1).} Four real high-resolution datasets, $n{=}10$,
mean $\pm$ std. Both methods are $100\%$ valid. Only FSG-LTN-GAN recovers the distribution of
the constraint margin, while per-property moment error (lower better) is
comparable, so the gap is invisible to it. $^{\dagger}$paired Wilcoxon $p<0.01$.}
\label{tab:rq1}
\small
\setlength{\tabcolsep}{5pt}
\begin{tabular}{lccccc}
\toprule
 & & \multicolumn{2}{c}{margin KS $\downarrow$} & \multicolumn{2}{c}{per-prop.\ moment $\downarrow$} \\
\cmidrule(lr){3-4}\cmidrule(lr){5-6}
Dataset & $R$ & C-DGM & \textbf{FSG-LTN-GAN} & C-DGM & FSG-LTN-GAN \\
\midrule
Alchemy        & $3{\times}10^4$ to $7{\times}10^6$ & $1.000\pm0.000$ & $\mathbf{0.040\pm0.007}^{\dagger}$ & $0.472\pm0.012$ & $0.327\pm0.058$ \\
tmQM           & $2.3\times10^4$  & $1.000\pm0.000$ & $\mathbf{0.045\pm0.010}^{\dagger}$ & $0.183\pm0.014$ & $0.192\pm0.025$ \\
Transition1x   & $10^3$           & $0.999\pm0.000$ & $\mathbf{0.065\pm0.010}^{\dagger}$ & $0.115\pm0.036$ & $0.196\pm0.038$ \\
Taxi           & $3.6$ to $1.1{\times}10^3$ & $0.574\pm0.005$ & $\mathbf{0.099\pm0.010}^{\dagger}$ & $0.492\pm0.024$ & $0.282\pm0.032$ \\
\bottomrule
\end{tabular}
\end{table}

\begin{figure}[t]\centering
\includegraphics[width=0.8\linewidth]{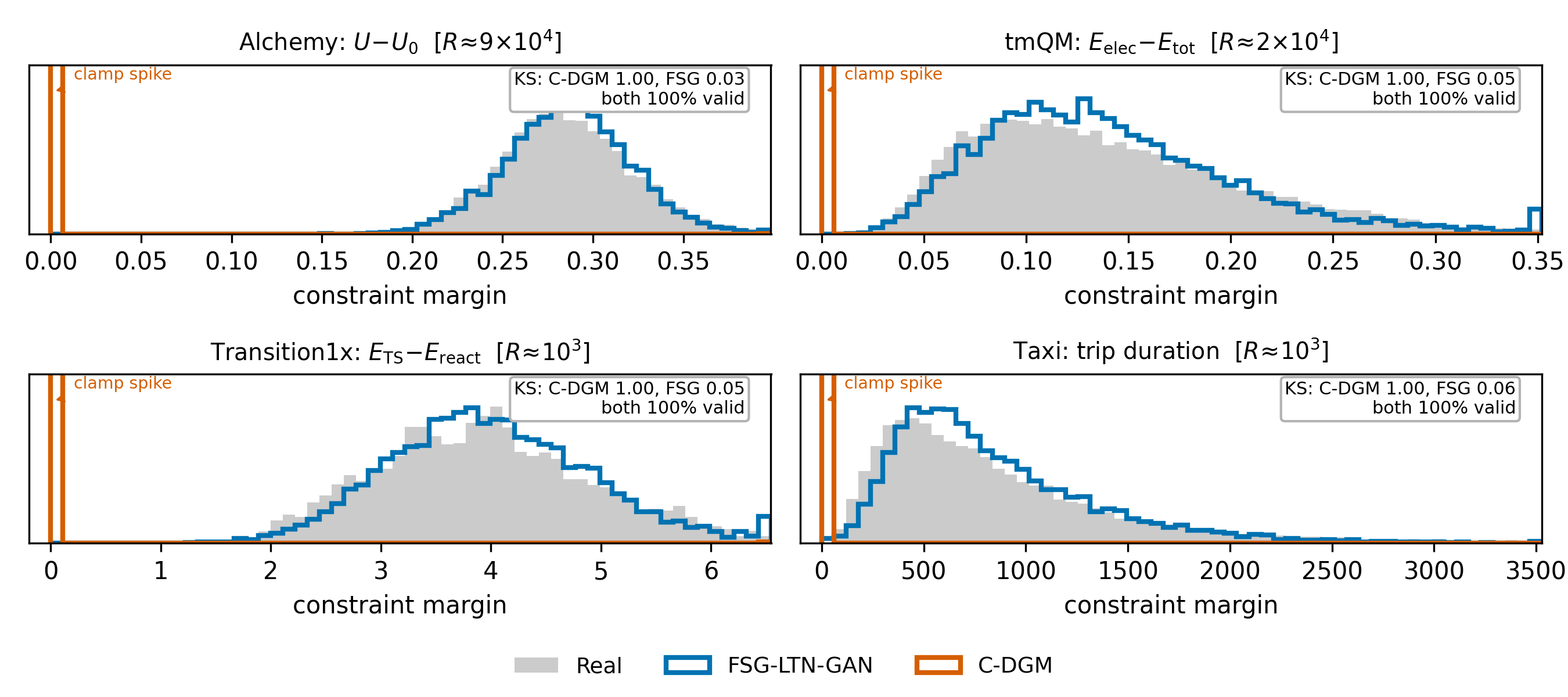}
\caption{\textbf{The distribution of the constraint margin on all four datasets.} Real margin
(grey) against C-DGM (orange, the constraint layer) and FSG-LTN-GAN (blue). All are $100\%$
valid. The clamp collapses every margin to a boundary point mass. FSG-LTN-GAN reproduces the
real distribution. The annotated KS is for the displayed high-$R$ margin. Table~\ref{tab:rq1}
averages each dataset's constraints, so its Taxi value also counts the low-$R$
$\mathrm{total}>\mathrm{fare}$ margin, which the clamp handles well.}
\label{fig:margins}
\end{figure}

Throughout, the constraint layer runs from the code released by \citet{stoian2024cdgm}, with
corrections that only help it
(including one for a latent bug in its Fourier-Motzkin reduction; Appendix~\ref{app:fair}).
Table~\ref{tab:rq1} shows that on all four datasets the constraint layer and FSG-LTN-GAN are both
$100\%$ valid, yet the constraint layer's margin KS is $0.57$ to $1.00$ while FSG-LTN-GAN's is
$0.04$ to $0.10$ ($p<0.01$, all ten seeds). Inspecting the raw margins
(Figure~\ref{fig:margins}) explains the numbers. The clamp places every sample's margin at the
boundary (about $10^{-6}$ at every quantile), a point mass disjoint from the real distribution
(the collapse of Proposition~\ref{thm:collapse}(i)). FSG-LTN-GAN's margin quantiles track the
real ones. The failure is invisible to the standard metrics we report. Per-property moment error is
comparable across methods (the constraint layer is even \emph{better} on tmQM and Transition1x),
so one can match every feature's marginal while losing the distribution of the constrained quantity. The clamp's
distortion tracks $R$ per margin (KS near $1$ on the high-$R$ margins, e.g.\ Alchemy $U-U_0$ at
$R\approx10^5$). Taxi's lower table value averages in a low-$R$ margin the clamp handles well, so
the effect tracks $R$ rather than the domain. A final, pre-registered test on a fifth
dataset outside both papers' suites (nycflights13) confirmed every prediction
(Appendix~\ref{app:extra}, Table~\ref{tab:flights}).

\paragraph{Overall sample quality is preserved.} On the
constraint layer's own density and coverage metrics \citep{naeem2020reliable}, function-symbol
grounding is close on the chemistry sets (density modestly favours the constraint layer)
and substantially exceeds it on Transition1x and Taxi,
where clamping collapses the joint distribution (coverage $0.20$ and $0.70$ against $0.05$;
Appendix~\ref{app:extra}).

\subsection{RQ2: the effect persists across the architectures we test}\label{sec:rq2}

\begin{table}[t]\centering
\caption{\textbf{Architecture generality (RQ2).} Margin KS $\downarrow$ ($n{=}10$, mean $\pm$ std)
for CTGAN$+$CL, TVAE$+$CL, the constraint layer on our MLP generator (C-DGM), and FSG-LTN-GAN.
All are $100\%$ valid. The constraint layer misses the distribution of the constraint margin on
all three generator families.}
\label{tab:rq2}
\small
\begin{tabular}{lcccc}
\toprule
Dataset & CTGAN$+$CL & TVAE$+$CL & C-DGM (MLP) & \textbf{FSG-LTN-GAN} \\
\midrule
Alchemy      & $0.674\pm0.036$ & $0.651\pm0.015$ & $1.000\pm0.000$ & $\mathbf{0.040\pm0.007}$ \\
tmQM         & $0.579\pm0.052$ & $0.517\pm0.011$ & $1.000\pm0.000$ & $\mathbf{0.045\pm0.010}$ \\
Transition1x & $0.592\pm0.040$ & $0.510\pm0.053$ & $0.999\pm0.000$ & $\mathbf{0.065\pm0.010}$ \\
Taxi         & $0.419\pm0.024$ & $0.348\pm0.029$ & $0.574\pm0.005$ & $\mathbf{0.099\pm0.010}$ \\
\bottomrule
\end{tabular}
\end{table}

Table~\ref{tab:rq2} applies the constraint layer to the two
most-cited tabular generators, CTGAN and TVAE. Both reach $100\%$ validity but leave margin
KS at $0.35$ to $0.67$, against $0.04$ to $0.10$ for FSG-LTN-GAN. The in-loop constraint layer on
our MLP generator is worse still ($0.57$ to $1.00$), consistent with the in-loop clamp removing
the generator's incentive to place the sub-resolution margin.
The method transfers across architectures the same way the failure does. Because the chart is a
change of coordinates on the data, an \emph{unmodified} CTGAN or TVAE trained in chart
coordinates and decoded through $\varphi$ is exactly valid and recovers the margin
(Appendix~\ref{app:extra}, Table~\ref{tab:backbones}).

\subsection{RQ3: the hybrid matches or exceeds the constraint layer on the Stoian et al.\ (2024) benchmark}\label{sec:rq3}

\begin{table}[t]\centering
\caption{\textbf{The constraint-layer benchmark of \citet{stoian2024cdgm} (RQ3).} Margin KS $\downarrow$,
$n{=}10$, mean $\pm$ std, all methods $100\%$ valid. $R_{\max}$ is the dataset's largest
per-constraint $R$. The hybrid outperforms the constraint layer
(paired Wilcoxon) on faults ($p{=}0.002$, the high-$R$ case), url ($p{=}0.03$), and wids
($p{=}0.004$), ties on heloc, lcld, news, and never underperforms it. FSG-all is worse than the
constraint layer everywhere but faults, showing the selective in-loop clamp is necessary.}
\label{tab:rq3}
\small
\begin{tabular}{lcccc}
\toprule
Dataset & $R_{\max}$ & C-DGM & FSG-all & \textbf{hybrid} \\
\midrule
faults & $1.8\times10^3$ & $0.745\pm0.014$ & $0.117\pm0.008$ & $\mathbf{0.117\pm0.008}$ \\
heloc  & $2$             & $0.156\pm0.012$ & $0.208\pm0.005$ & $0.158\pm0.014$ \\
lcld   & $1$             & $0.131\pm0.008$ & $0.657\pm0.004$ & $0.133\pm0.007$ \\
url    & $1$             & $0.159\pm0.010$ & $0.196\pm0.018$ & $\mathbf{0.144\pm0.021}$ \\
news   & $2$             & $0.294\pm0.008$ & $0.318\pm0.007$ & $0.297\pm0.007$ \\
wids   & $4$             & $0.140\pm0.007$ & $0.319\pm0.004$ & $\mathbf{0.133\pm0.008}$ \\
\bottomrule
\end{tabular}
\end{table}

\emph{Does the method help on the constraint layer's own benchmark? Yes, wherever $R$ is high.}
Table~\ref{tab:rq3} reports all six datasets of \citet{stoian2024cdgm},
predominantly low-$R$ tabular-ML tasks where clamping is already adequate. On the one
high-$R$ dataset, faults (bounding-box orderings on large sensor coordinates,
$R_{\max}\approx1.8\times10^3$), the hybrid charts every bounded variable, coinciding with FSG-all,
and improves ($p{=}0.002$, all ten seeds). On the low-$R$
remainder the hybrid clamps and matches the constraint layer, for three wins, three ties,
and zero losses across the benchmark, all at $100\%$ validity. The FSG-all ablation, by contrast, is worse
than the constraint layer everywhere but faults, since charting a point-mass or sub-resolution margin
smears it. The hybrid's per-constraint selection
(Section~\ref{sec:hybrid}) is therefore necessary, and $R$ predicts which datasets
it helps before training. The selection is robust to its two thresholds. Only extreme grid
corners that chart most of wids' constraints erode that dataset's win
(Appendix~\ref{app:extra}, Table~\ref{tab:threshsens}).

\subsection{RQ4: predicate grounding does not reach high-$R$ constraints}\label{sec:rq4}
Corollary~\ref{cor:ordering} bounds a predicate grounding's usable-gradient fraction
at $\Theta(1/R)$. We test three placements of the same well-scaled predicate (the
generator loss, i.e.\ G-LTN-GAN; discriminator re-weighting; discriminator-feature
augmentation). We score the per-ordering satisfaction fraction, averaged over the dataset's orderings ($n{=}10$,
mean $\pm$ std): $0.58\pm0.02$, $0.51\pm0.03$, $0.14\pm0.08$ on Alchemy, where a free
generator with no mechanism scores $0.49\pm0.03$, and $0.59\pm0.03$, $0.55\pm0.08$,
$0.29\pm0.07$ on tmQM (free generator $0.53\pm0.08$), against $1.000\pm0.000$ for
function-symbol grounding. No placement helps by more than $0.09$. The failure tracks the
coordinates, not the placement (Table~\ref{tab:rq4}).

\section{Related Work}\label{sec:related}
\textbf{Neuro-symbolic generation and LTNs.} Logic Tensor Networks train by maximizing grounded
satisfaction \citep{badreddine2022ltn}, and related work realises logic as a differentiable loss
\citep{xu2018semantic,fischer2019dl2}. All are \emph{predicate-style} soft constraints (RQ4).
Hard-constraint output layers guarantee
specific fragments \citep{ahmed2022spl,hoernle2022multiplexnet,giunchiglia2020coherent}. \textbf{Constrained tabular
generation.} The constraint layer of \citet{stoian2024cdgm}, our primary baseline, guarantees
linear inequalities by clamping. GOGGLE \citep{liu2022goggle} injects only simple correlations,
and CTGAN and TVAE \citep{xu2019modeling} give no guarantees. Closest are constrained adversarial
networks \citep{di2020efficient} and a non-convex constraint layer \citep{stoianbeyond};
Appendix~\ref{app:related} surveys the rest and our function symbols'
antecedents \citep{dugas2009incorporating,carpenter2017stan}.

\section{Discussion and Conclusions}\label{sec:conc}
In an LTN-GAN the choice of grounding is decisive. Predicate grounding cannot reach
high-resolution constraints, and clamping is valid but collapses the margin distribution,
invisibly to standard metrics and predictably from $R$. Function-symbol grounding recovers that
distribution with exact validity, outperforming the constraint layer on four high-resolution
datasets, and the hybrid at least matches it on the benchmark of \citet{stoian2024cdgm},
with the same advantage in conditional inverse design (Appendix~\ref{app:extra}).

\textbf{Scope and limitations.} Our charts cover the linear fragment, where Fourier-Motzkin
gives the admissible intervals. Non-convex feasible sets from nonlinear or disjunctive
constraints are left to future work, and discrete margins are boundary point masses
the hybrid correctly clamps. 

\bibliography{nesy2026-sample}

@article{badreddine2022ltn,
  author  = {Badreddine, Samy and d'Avila Garcez, Artur and Serafini, Luciano and Spranger, Michael},
  title   = {{Logic Tensor Networks}},
  journal = {Artificial Intelligence},
  volume  = {303},
  pages   = {103649},
  year    = {2022},
  doi     = {10.1016/j.artint.2021.103649},
  url     = {https://doi.org/10.1016/j.artint.2021.103649}
}

@article{upreti2026ltngan,
  author    = {Upreti, Nijesh and Belle, Vaishak},
  title     = {{Logic Tensor Network-Enhanced Generative Adversarial Network}},
  journal   = {Electronic Proceedings in Theoretical Computer Science},
  volume    = {439},
  pages     = {89--113},
  year      = {2026},
  publisher = {Open Publishing Association},
  doi       = {10.4204/EPTCS.439.8},
  url       = {https://doi.org/10.4204/EPTCS.439.8}
}

@inproceedings{xu2018semantic,
  author    = {Xu, Jingyi and Zhang, Zilu and Friedman, Tal and Liang, Yitao and {Van den Broeck}, Guy},
  title     = {A Semantic Loss Function for Deep Learning with Symbolic Knowledge},
  booktitle = {Proceedings of the 35th International Conference on Machine Learning (ICML)},
  series    = {Proceedings of Machine Learning Research},
  volume    = {80},
  pages     = {5502--5511},
  year      = {2018},
  url       = {https://proceedings.mlr.press/v80/xu18h.html}
}

@inproceedings{stoian2024cdgm,
  author    = {Stoian, Mihaela C{\u{a}}t{\u{a}}lina and Dyrmishi, Salijona and Cordy, Maxime and Lukasiewicz, Thomas and Giunchiglia, Eleonora},
  title     = {How Realistic Is Your Synthetic Data? Constraining Deep Generative Models for Tabular Data},
  booktitle = {Proceedings of the 12th International Conference on Learning Representations (ICLR)},
  year      = {2024},
  url       = {https://openreview.net/forum?id=tBROYsEz9G}
}

@inproceedings{ahmed2022spl,
  author    = {Ahmed, Kareem and Teso, Stefano and Chang, Kai-Wei and {Van den Broeck}, Guy and Vergari, Antonio},
  title     = {Semantic Probabilistic Layers for Neuro-Symbolic Learning},
  booktitle = {Advances in Neural Information Processing Systems (NeurIPS)},
  volume    = {35},
  year      = {2022},
  url       = {https://proceedings.neurips.cc/paper_files/paper/2022/hash/c182ec594f38926b7fcb827635b9a8f4-Abstract-Conference.html}
}

@inproceedings{fischer2019dl2,
  author    = {Fischer, Marc and Balunovi{\'c}, Mislav and Drachsler-Cohen, Dana and Gehr, Timon and Zhang, Ce and Vechev, Martin},
  title     = {{DL2}: Training and Querying Neural Networks with Logic},
  booktitle = {Proceedings of the 36th International Conference on Machine Learning (ICML)},
  series    = {Proceedings of Machine Learning Research},
  volume    = {97},
  pages     = {1931--1941},
  year      = {2019},
  url       = {https://proceedings.mlr.press/v97/fischer19a.html}
}

@inproceedings{hoernle2022multiplexnet,
  author    = {Hoernle, Nick and Karampatsis, Rafael-Michael and Belle, Vaishak and Gal, Kobi},
  title     = {{MultiplexNet}: Towards Fully Satisfied Logical Constraints in Neural Networks},
  booktitle = {Proceedings of the 36th AAAI Conference on Artificial Intelligence},
  pages     = {5700--5709},
  year      = {2022},
  doi       = {10.1609/aaai.v36i5.20512},
  url       = {https://doi.org/10.1609/aaai.v36i5.20512}
}

@inproceedings{giunchiglia2020coherent,
  author    = {Giunchiglia, Eleonora and Lukasiewicz, Thomas},
  title     = {Coherent Hierarchical Multi-Label Classification Networks},
  booktitle = {Advances in Neural Information Processing Systems (NeurIPS)},
  volume    = {33},
  year      = {2020},
  url       = {https://proceedings.neurips.cc/paper/2020/hash/6dd4e10e3296fa63738371ec0d5df818-Abstract.html}
}

@article{lagaris1998ann,
  author  = {Lagaris, Isaac E. and Likas, Aristidis and Fotiadis, Dimitrios I.},
  title   = {Artificial Neural Networks for Solving Ordinary and Partial Differential Equations},
  journal = {IEEE Transactions on Neural Networks},
  volume  = {9},
  number  = {5},
  pages   = {987--1000},
  year    = {1998},
  doi     = {10.1109/72.712178},
  url     = {https://doi.org/10.1109/72.712178}
}

@article{lu2021hardconstraints,
  author  = {Lu, Lu and Pestourie, Rapha{\"e}l and Yao, Wenjie and Wang, Zhicheng and Verdugo, Francesc and Johnson, Steven G.},
  title   = {Physics-Informed Neural Networks with Hard Constraints for Inverse Design},
  journal = {SIAM Journal on Scientific Computing},
  volume  = {43},
  number  = {6},
  pages   = {B1105--B1132},
  year    = {2021},
  doi     = {10.1137/21M1397908},
  url     = {https://doi.org/10.1137/21M1397908}
}

@article{chen2019alchemy,
  author  = {Chen, Guangyong and Chen, Pengfei and Hsieh, Chang-Yu and Lee, Chee-Kong and Liao, Benben and Liao, Renjie and Liu, Weiwen and Qiu, Jiezhong and Sun, Qiming and Tang, Jie and Zemel, Richard and Zhang, Shengyu},
  title   = {Alchemy: A Quantum Chemistry Dataset for Benchmarking {AI} Models},
  journal = {arXiv preprint arXiv:1906.09427},
  year    = {2019},
  doi     = {10.48550/arXiv.1906.09427},
  url     = {https://arxiv.org/abs/1906.09427}
}

@inproceedings{naeem2020reliable,
  author    = {Naeem, Muhammad Ferjad and Oh, Seong Joon and Uh, Youngjung and Choi, Yunjey and Yoo, Jaejun},
  title     = {Reliable Fidelity and Diversity Metrics for Generative Models},
  booktitle = {Proceedings of the 37th International Conference on Machine Learning (ICML)},
  series    = {Proceedings of Machine Learning Research},
  volume    = {119},
  pages     = {7176--7185},
  year      = {2020},
  url       = {https://proceedings.mlr.press/v119/naeem20a.html}
}

@inproceedings{xu2019modeling,
  author    = {Xu, Lei and Skoularidou, Maria and Cuesta-Infante, Alfredo and Veeramachaneni, Kalyan},
  title     = {Modeling Tabular Data using Conditional {GAN}},
  booktitle = {Advances in Neural Information Processing Systems (NeurIPS)},
  volume    = {32},
  pages     = {7333--7343},
  year      = {2019},
  url       = {https://proceedings.neurips.cc/paper/2019/hash/254ed7d2de3b23ab10936522dd547b78-Abstract.html}
}

@inproceedings{liu2022goggle,
  author    = {Liu, Tennison and Qian, Zhaozhi and Berrevoets, Jeroen and van der Schaar, Mihaela},
  title     = {{GOGGLE}: Generative Modelling for Tabular Data by Learning Relational Structure},
  booktitle = {Proceedings of the 11th International Conference on Learning Representations (ICLR)},
  year      = {2023},
  url       = {https://openreview.net/forum?id=fPVRcJqspu}
}

@article{balcells2020tmqm,
  author  = {Balcells, David and Skjelstad, Bastian Bjerkem},
  title   = {{tmQM} Dataset---Quantum Geometries and Properties of 86k Transition Metal Complexes},
  journal = {Journal of Chemical Information and Modeling},
  volume  = {60},
  number  = {12},
  pages   = {6135--6146},
  year    = {2020},
  doi     = {10.1021/acs.jcim.0c01041},
  url     = {https://doi.org/10.1021/acs.jcim.0c01041}
}

@article{schreiner2022transition1x,
  author  = {Schreiner, Mathias and Bhowmik, Arghya and Vegge, Tejs and Busk, Jonas and Winther, Ole},
  title   = {{Transition1x}---a Dataset for Building Generalizable Reactive Machine Learning Potentials},
  journal = {Scientific Data},
  volume  = {9},
  pages   = {779},
  year    = {2022},
  doi     = {10.1038/s41597-022-01870-w},
  url     = {https://doi.org/10.1038/s41597-022-01870-w}
}

@misc{nyctlc,
  author       = {{New York City Taxi and Limousine Commission}},
  title        = {{TLC} Trip Record Data},
  year         = {2024},
  howpublished = {\url{https://www.nyc.gov/site/tlc/about/tlc-trip-record-data.page}},
  note         = {Accessed: 2026-06-16}
}

@inproceedings{serafini2016logic,
  author    = {Serafini, Luciano and d'Avila Garcez, Artur},
  title     = {Learning and Reasoning with {Logic Tensor Networks}},
  booktitle = {AI*IA 2016: Advances in Artificial Intelligence},
  series    = {Lecture Notes in Computer Science},
  volume    = {10037},
  pages     = {334--348},
  publisher = {Springer},
  year      = {2016},
  url       = {https://doi.org/10.1007/978-3-319-49130-1_25}
}

@inproceedings{di2020efficient,
  author    = {Di Liello, Luca and Ardino, Pierfrancesco and Gobbi, Jacopo and Morettin, Paolo and Teso, Stefano and Passerini, Andrea},
  title     = {Efficient Generation of Structured Objects with Constrained Adversarial Networks},
  booktitle = {Advances in Neural Information Processing Systems (NeurIPS)},
  volume    = {33},
  year      = {2020},
  url       = {https://proceedings.neurips.cc/paper/2020/hash/a87c11b9100c608b7f8e98cfa316ff7b-Abstract.html}
}

@inproceedings{hu2018deep,
  author    = {Hu, Zhiting and Yang, Zichao and Salakhutdinov, Ruslan and Qin, Lianhui and Liang, Xiaodan and Dong, Haoye and Xing, Eric P.},
  title     = {Deep Generative Models with Learnable Knowledge Constraints},
  booktitle = {Advances in Neural Information Processing Systems (NeurIPS)},
  volume    = {31},
  pages     = {10522--10533},
  year      = {2018},
  url       = {https://proceedings.neurips.cc/paper/2018/hash/d7e77c835af3d2a803c1cf28d60575bc-Abstract.html}
}

@inproceedings{stoianbeyond,
  author    = {Stoian, Mihaela C{\u{a}}t{\u{a}}lina and Giunchiglia, Eleonora},
  title     = {Beyond the Convexity Assumption: Realistic Tabular Data Generation under Quantifier-Free Real Linear Constraints},
  booktitle = {Proceedings of the 13th International Conference on Learning Representations (ICLR)},
  year      = {2025},
  url       = {https://openreview.net/forum?id=rx0TCew0Lj}
}

@inproceedings{peng2023generating,
  author    = {Peng, Yifei and Zha, Zijie and Jin, Yu and Luo, Zhexu and Dai, Wang-Zhou and Ren, Zhong and Ding, Yao-Xiang and Zhou, Kun},
  title     = {Generating by Understanding: Neural Visual Generation with Logical Symbol Groundings},
  booktitle = {Proceedings of the 31st ACM SIGKDD Conference on Knowledge Discovery and Data Mining},
  pages     = {2291--2302},
  year      = {2025},
  doi       = {10.1145/3711896.3736978},
  url       = {https://doi.org/10.1145/3711896.3736978}
}

@inproceedings{young2022neurosymbolic,
  author    = {Young, Halley and Du, Maxwell and Bastani, Osbert},
  title     = {Neurosymbolic Deep Generative Models for Sequence Data with Relational Constraints},
  booktitle = {Advances in Neural Information Processing Systems (NeurIPS)},
  volume    = {35},
  year      = {2022},
  url       = {http://papers.nips.cc/paper_files/paper/2022/hash/f13ceb1b94145aad0e54186373cc86d7-Abstract-Conference.html}
}

@inproceedings{misino2022vael,
  author    = {Misino, Eleonora and Marra, Giuseppe and Sansone, Emanuele},
  title     = {{VAEL}: Bridging Variational Autoencoders and Probabilistic Logic Programming},
  booktitle = {Advances in Neural Information Processing Systems (NeurIPS)},
  volume    = {35},
  year      = {2022},
  url       = {http://papers.nips.cc/paper_files/paper/2022/hash/1e38b2a0b77541b14a3315c99697b835-Abstract-Conference.html}
}

@inproceedings{ahmed2023pseudo,
  author    = {Ahmed, Kareem and Chang, Kai-Wei and {Van den Broeck}, Guy},
  title     = {A Pseudo-Semantic Loss for Autoregressive Models with Logical Constraints},
  booktitle = {Advances in Neural Information Processing Systems (NeurIPS)},
  volume    = {36},
  year      = {2023},
  url       = {http://papers.nips.cc/paper_files/paper/2023/hash/3accfe8332366a6f740d8740cd4cd653-Abstract-Conference.html}
}

@inproceedings{yang2022injecting,
  author    = {Yang, Zhun and Lee, Joohyung and Park, Chiyoun},
  title     = {Injecting Logical Constraints into Neural Networks via Straight-Through Estimators},
  booktitle = {Proceedings of the 39th International Conference on Machine Learning (ICML)},
  series    = {Proceedings of Machine Learning Research},
  volume    = {162},
  pages     = {25096--25122},
  year      = {2022},
  url       = {https://proceedings.mlr.press/v162/yang22h.html}
}

@inproceedings{mendez2024semantic,
  author    = {M{\'e}ndez-Lucero, Miguel {\'A}ngel and Bojorquez Gallardo, Enrique and Belle, Vaishak},
  title     = {Semantic Objective Functions: A Distribution-Aware Method for Adding Logical Constraints in Deep Learning},
  booktitle = {Proceedings of the 17th International Conference on Agents and Artificial Intelligence (ICAART)},
  pages     = {909--917},
  year      = {2025},
  doi       = {10.5220/0013229200003890},
  url       = {https://doi.org/10.5220/0013229200003890}
}

@inproceedings{li2023neuro,
  author    = {Li, Zenan and Huang, Yunpeng and Li, Zhaoyu and Yao, Yuan and Xu, Jingwei and Chen, Taolue and Ma, Xiaoxing and Lu, Jian},
  title     = {Neuro-symbolic Learning Yielding Logical Constraints},
  booktitle = {Advances in Neural Information Processing Systems (NeurIPS)},
  volume    = {36},
  year      = {2023},
  url       = {http://papers.nips.cc/paper_files/paper/2023/hash/4459c3c143db74ee52afebdf56836375-Abstract-Conference.html}
}

@article{chao2021constrained,
  author  = {Chao, Xiaopeng and Cao, Jiangzhong and Lu, Yuqin and Dai, Qingyun and Liang, Shangsong},
  title   = {Constrained Generative Adversarial Networks},
  journal = {IEEE Access},
  volume  = {9},
  pages   = {19208--19218},
  year    = {2021},
  doi     = {10.1109/ACCESS.2021.3054822},
  url     = {https://doi.org/10.1109/ACCESS.2021.3054822}
}

@inproceedings{heim2019interactive,
  author    = {Heim, Eric},
  title     = {Constrained Generative Adversarial Networks for Interactive Image Generation},
  booktitle = {IEEE/CVF Conference on Computer Vision and Pattern Recognition (CVPR)},
  pages     = {10753--10761},
  year      = {2019},
  doi       = {10.1109/CVPR.2019.01101},
  url       = {https://doi.org/10.1109/CVPR.2019.01101}
}

@article{hess2022physically,
  author  = {Hess, Philipp and Dr{\"u}ke, Markus and Petri, Stefan and Strnad, Felix M. and Boers, Niklas},
  title   = {Physically Constrained Generative Adversarial Networks for Improving Precipitation Fields from {Earth} System Models},
  journal = {Nature Machine Intelligence},
  volume  = {4},
  number  = {10},
  pages   = {828--839},
  year    = {2022},
  doi     = {10.1038/s42256-022-00540-1},
  url     = {https://doi.org/10.1038/s42256-022-00540-1}
}

@article{zeng2021enforcing,
  author  = {Zeng, Yang and Wu, Jin-Long and Xiao, Heng},
  title   = {Enforcing Imprecise Constraints on Generative Adversarial Networks for Emulating Physical Systems},
  journal = {Communications in Computational Physics},
  volume  = {30},
  number  = {3},
  pages   = {635--665},
  year    = {2021},
  doi     = {10.4208/cicp.OA-2020-0106},
  url     = {https://doi.org/10.4208/cicp.OA-2020-0106}
}

@inproceedings{xue2019embedding,
  author    = {Xue, Yexiang and van Hoeve, Willem-Jan},
  title     = {Embedding Decision Diagrams into Generative Adversarial Networks},
  booktitle = {Integration of Constraint Programming, Artificial Intelligence, and Operations Research (CPAIOR)},
  series    = {Lecture Notes in Computer Science},
  volume    = {11494},
  pages     = {616--632},
  publisher = {Springer},
  year      = {2019},
  url       = {https://doi.org/10.1007/978-3-030-19212-9_41}
}

@inproceedings{seff2022vitruvion,
  author    = {Seff, Ari and Zhou, Wenda and Richardson, Nick and Adams, Ryan P.},
  title     = {Vitruvion: A Generative Model of Parametric {CAD} Sketches},
  booktitle = {Proceedings of the 10th International Conference on Learning Representations (ICLR)},
  year      = {2022},
  url       = {https://openreview.net/forum?id=Ow1C7s3UcY}
}

@inproceedings{para2021sketchgen,
  author    = {Para, Wamiq Reyaz and Bhat, Shariq Farooq and Guerrero, Paul and Kelly, Tom and Mitra, Niloy J. and Guibas, Leonidas J. and Wonka, Peter},
  title     = {{SketchGen}: Generating Constrained {CAD} Sketches},
  booktitle = {Advances in Neural Information Processing Systems (NeurIPS)},
  volume    = {34},
  pages     = {5077--5088},
  year      = {2021},
  url       = {https://proceedings.neurips.cc/paper/2021/hash/28891cb4ab421830acc36b1f5fd6c91e-Abstract.html}
}

@inproceedings{ferber2024genco,
  author    = {Ferber, Aaron M. and Zharmagambetov, Arman and Huang, Taoan and Dilkina, Bistra and Tian, Yuandong},
  title     = {{GenCO}: Generating Diverse Designs with Combinatorial Constraints},
  booktitle = {Proceedings of the 41st International Conference on Machine Learning (ICML)},
  series    = {Proceedings of Machine Learning Research},
  volume    = {235},
  pages     = {13445--13459},
  year      = {2024},
  url       = {https://proceedings.mlr.press/v235/ferber24a.html}
}

@inproceedings{liu2020chance,
  author    = {Liu, Xianggen and Liu, Qiang and Song, Sen and Peng, Jian},
  title     = {A Chance-Constrained Generative Framework for Sequence Optimization},
  booktitle = {Proceedings of the 37th International Conference on Machine Learning (ICML)},
  series    = {Proceedings of Machine Learning Research},
  volume    = {119},
  pages     = {6271--6281},
  year      = {2020},
  url       = {http://proceedings.mlr.press/v119/liu20i.html}
}

@inproceedings{li2020supporting,
  author    = {Li, Wanxin},
  title     = {Supporting Database Constraints in Synthetic Data Generation Based on Generative Adversarial Networks},
  booktitle = {Proceedings of the 2020 ACM SIGMOD International Conference on Management of Data},
  pages     = {2875--2877},
  year      = {2020},
  doi       = {10.1145/3318464.3384414},
  url       = {https://doi.org/10.1145/3318464.3384414}
}

@inproceedings{kotelnikov2023tabddpm,
  author    = {Kotelnikov, Akim and Baranchuk, Dmitry and Rubachev, Ivan and Babenko, Artem},
  title     = {{TabDDPM}: Modelling Tabular Data with Diffusion Models},
  booktitle = {Proceedings of the 40th International Conference on Machine Learning (ICML)},
  series    = {Proceedings of Machine Learning Research},
  volume    = {202},
  pages     = {17564--17579},
  year      = {2023},
  url       = {https://proceedings.mlr.press/v202/kotelnikov23a.html}
}

@inproceedings{kim2023stasy,
  author    = {Kim, Jayoung and Lee, Chaejeong and Park, Noseong},
  title     = {{STaSy}: Score-based Tabular Data Synthesis},
  booktitle = {Proceedings of the 11th International Conference on Learning Representations (ICLR)},
  year      = {2023},
  url       = {https://openreview.net/forum?id=1mNssCWt_v}
}

@inproceedings{zhao2021ctabgan,
  author    = {Zhao, Zilong and Kunar, Aditya and Birke, Robert and Chen, Lydia Y.},
  title     = {{CTAB-GAN}: Effective Table Data Synthesizing},
  booktitle = {Proceedings of the 13th Asian Conference on Machine Learning (ACML)},
  series    = {Proceedings of Machine Learning Research},
  volume    = {157},
  pages     = {97--112},
  year      = {2021},
  url       = {https://proceedings.mlr.press/v157/zhao21a.html}
}

@article{dugas2009incorporating,
  author  = {Dugas, Charles and Bengio, Yoshua and B{\'e}lisle, Fran{\c{c}}ois and Nadeau, Claude and Garcia, Ren{\'e}},
  title   = {Incorporating Functional Knowledge in Neural Networks},
  journal = {Journal of Machine Learning Research},
  volume  = {10},
  number  = {42},
  pages   = {1239--1262},
  year    = {2009}
}

@article{carpenter2017stan,
  author  = {Carpenter, Bob and Gelman, Andrew and Hoffman, Matthew D. and Lee, Daniel and Goodrich, Ben and Betancourt, Michael and Brubaker, Marcus and Guo, Jiqiang and Li, Peter and Riddell, Allen},
  title   = {Stan: A Probabilistic Programming Language},
  journal = {Journal of Statistical Software},
  volume  = {76},
  number  = {1},
  pages   = {1--32},
  year    = {2017},
  doi     = {10.18637/jss.v076.i01}
}

@article{ramakrishnan2015delta,
  author  = {Ramakrishnan, Raghunathan and Dral, Pavlo O. and Rupp, Matthias and von Lilienfeld, O. Anatole},
  title   = {Big Data Meets Quantum Chemistry Approximations: The $\Delta$-Machine Learning Approach},
  journal = {Journal of Chemical Theory and Computation},
  volume  = {11},
  number  = {5},
  pages   = {2087--2096},
  year    = {2015},
  doi     = {10.1021/acs.jctc.5b00099}
}

\clearpage
\appendix
\raggedbottom
\let\bodysection\section
\renewcommand{\section}{\clearpage\bodysection}
\section{Extended Related Work}\label{app:related}
This appendix expands Section~\ref{sec:related}. The body cites only the most directly related work.

\paragraph{Logic as a soft training signal.} Beyond Logic Tensor Networks
\citep{badreddine2022ltn,serafini2016logic}, many methods inject logic as a differentiable loss. The
semantic loss \citep{xu2018semantic} and its pseudo-semantic extension for autoregressive models
\citep{ahmed2023pseudo} penalise probability mass on violating assignments. DL2 \citep{fischer2019dl2},
straight-through estimators \citep{yang2022injecting}, distribution-aware objectives
\citep{mendez2024semantic}, and bilevel formulations \citep{li2023neuro} optimise logical losses
directly, and VAEL \citep{misino2022vael} couples a variational autoencoder with probabilistic logic.
These all inject logic as a \emph{soft} satisfaction signal rather than a hard guarantee. We do
not evaluate each of them on high-resolution structural constraints, because the obstruction we
identify is a property of soft satisfaction itself. In the ambient coordinates, a structural
margin predicate has usable gradient on a $\Theta(1/R)$ fraction of samples
(Corollary~\ref{cor:ordering}), and RQ4 confirms this for the three predicate placements we
test.

\paragraph{Hard constraints by construction.} A second family of methods guarantees
satisfaction, each for a specific class of constraints. Semantic Probabilistic Layers
\citep{ahmed2022spl}, MultiplexNet \citep{hoernle2022multiplexnet}, and coherent hierarchical
classifiers \citep{giunchiglia2020coherent} cover classes of logical constraints, and
physics-informed networks enforce differential constraints
\citep{lagaris1998ann,lu2021hardconstraints}. For linear constraints on tabular data, the constraint
layer of \citet{stoian2024cdgm} clamps onto the feasible polytope, recently extended to quantifier-free
non-convex constraints \citep{stoianbeyond}. We share the exact-validity goal but, rather than clamping,
ground the constraint as a change of coordinates that can also preserve the distribution of the constraint margin.

\paragraph{Constrained generative models.} Many domains have built constraints into
generators. Constrained adversarial networks enforce logical requirements during training
\citep{di2020efficient,hu2018deep,chao2021constrained}, and related mechanisms appear in
interactive image editing \citep{heim2019interactive}, physical fields
\citep{hess2022physically,zeng2021enforcing}, decision diagrams \citep{xue2019embedding},
parametric CAD sketches \citep{seff2022vitruvion,para2021sketchgen}, combinatorial design
\citep{ferber2024genco}, sequence optimisation \citep{liu2020chance,young2022neurosymbolic},
and visual scene generation \citep{peng2023generating}. Almost all of them enforce the
constraints by penalty, rejection, or post-hoc projection. Function-symbol grounding instead
reparameterises the output space, so the constraints hold by construction. Monotone link
functions that enforce order and positivity in regression networks go back to
\citet{dugas2009incorporating}, and probabilistic-programming samplers routinely transform
positive, interval, and ordered parameters to unconstrained space before sampling
\citep{carpenter2017stan}. We use the same links, but in a new role. They ground logical
axioms, they are assembled along the Fourier-Motzkin order so that any satisfiable linear
system can be charted, and the discriminator is trained in the transformed coordinates.
Modelling a difference rather than the levels that form it also echoes $\Delta$-machine
learning for chemical energies \citep{ramakrishnan2015delta}, a supervised antecedent of the
margin-form baseline of Appendix~\ref{app:extra}.

\paragraph{Rejection sampling and chance-constrained formulations.} Two further strategies
deserve direct comparison. \emph{Rejection sampling} draws from an unconstrained generator and
keeps only the valid samples. It recovers the correct conditional distribution in principle, but
its acceptance rate is the unconstrained generator's own validity, which is exactly what
collapses in the high-$R$ regime (that validity is $0.05$ on Alchemy, Table~\ref{tab:rq1full}, and decreases
as $R$ grows), and for definitional identities it is zero outright. By
Corollary~\ref{cor:identity} an absolutely continuous generator satisfies an identity with
probability $0$, so no rejection budget suffices. \emph{Chance-constrained} formulations
\citep{liu2020chance} require constraints to hold with a prescribed probability rather than on
every sample. They operate in the same soft-satisfaction regime as predicate grounding and
inherit its high-$R$ gradient obstruction (Corollary~\ref{cor:ordering}). Function-symbol grounding
avoids both: validity is guaranteed, no rejection loop is needed, and the margin is learned at unit
scale.

\paragraph{Tabular data generation.} For tabular data, adversarial and variational generators (CTGAN and
TVAE \citep{xu2019modeling}, CTAB-GAN \citep{zhao2021ctabgan}), score-based and diffusion models (STaSy
\citep{kim2023stasy}, TabDDPM \citep{kotelnikov2023tabddpm}), and relational-structure models (GOGGLE
\citep{liu2022goggle}) improve marginal and dependency fidelity but provide no validity
guarantees. Database-constraint enforcement has also been added to GAN-based synthesis
\citep{li2020supporting}. The constraint layer \citep{stoian2024cdgm} and our function-symbol
grounding both add hard guarantees, but only function-symbol grounding recovers the
distribution of the constraint margin for high-resolution constraints.

\section{The Hybrid Procedure}\label{app:alg}
\SetKwComment{tcc}{$\triangleright$\ }{}
\SetKwComment{tcp}{$\triangleright$\ }{}
\SetCommentSty{textnormal}
\begin{algorithm2e}[H]
\DontPrintSemicolon
\SetKwInOut{Input}{Input}\SetKwInOut{Output}{Output}
\Input{linear constraints $\Pi$; data $\mathcal{D}$; thresholds $\tau_s{=}0.9$ (satisfaction), $\tau_d{=}0.2$ (discreteness)}
\Output{an LTN-GAN generator with exact validity on $\Pi$}
\BlankLine
\tcc{Phase 1: reduce $\Pi$ to per-variable interval bounds}
Fix a variable order; by Fourier-Motzkin elimination, write each $x_i$'s bounds $\ell_i,u_i$ as maxima and minima of linear functions of the strictly earlier variables\;
\BlankLine
\tcc{Phase 2: choose a grounding per bounded variable (one short pre-run)}
Train a free generator briefly; measure each bounded variable $x_i$'s interval satisfaction $s_i$ on the generator's samples and its binding-margin discreteness $d_i$ on $\mathcal{D}$\;
\For{each bounded variable $x_i$}{
  \eIf{$s_i<\tau_s$ \textnormal{\textbf{and}} $d_i\le\tau_d$}{
    $\mathrm{mode}_i \leftarrow \textsc{chart}$\tcp*{high-$R$, continuous}
  }{
    $\mathrm{mode}_i \leftarrow \textsc{clamp}$\tcp*{low-$R$ or discrete}
  }
}
\BlankLine
\tcc{Phase 3: train; charted variables' axioms hold by construction (no satisfaction loss)}
\For{each training step}{
  draw noise $\zeta$ and set the free coordinates $z \leftarrow G_\theta(\zeta)$\;
  \For{each variable $i$ in Fourier-Motzkin order}{
    \uIf{$x_i$ is fixed by an identity}{$x_i \leftarrow \textstyle\sum_{j<i} w_{ij}x_j+w_{i0}$\;}
    \uElseIf{$x_i$ is free (highest-order in no axiom)}{$x_i \leftarrow z_i$\;}
    \uElseIf{$\mathrm{mode}_i = \textsc{chart}$}{
      $x_i \leftarrow$ the link fitting $[\ell_i,u_i]$ (softplus or $\sigma$, Section~\ref{sec:method}; $\exp$, Appendix~\ref{app:hparams})\tcp*{$\mathrm{Sat}{=}1$}
    }
    \Else{$x_i \leftarrow \min(\max(z_i,\ell_i),\,u_i)$\tcp*{in-loop clamp}}
  }
  update $G_\theta,D_\psi$ adversarially on the chart features: each charted variable enters as its coordinate $z_i$ and each clamped variable as its standardized clamped value, with real samples encoded the same way through $\varphi^{-1}$\;
}
\caption{Hybrid function-symbol grounding in an LTN-GAN.}
\label{alg:hybrid}
\end{algorithm2e}

\section{Proof of Proposition~\ref{thm:valid}}\label{app:proof}
Fourier-Motzkin elimination over $\Pi$ in the variable order gives, for each $x_i$, its lower and
upper bounds $\ell_i,u_i$ as maxima and minima of linear functions of the strictly-earlier
variables. Assembling in
this order, $x_i$'s bounds depend only on finalized variables; if $\Pi$ is satisfiable the
interval $[\ell_i,u_i]$ is non-empty (Fourier-Motzkin completeness). The interior placement
also requires the \emph{open} interval to be non-empty, $\ell_i<u_i$: if $\Pi$ pinches
$\ell_i=u_i$ (an equality implied by the inequalities though not declared as an identity), the
variable is routed to the dependent branch below and set to its single admissible value. For a
bounded or one-sided
variable with a non-degenerate interval, softplus and $\sigma$ map $\mathbb{R}$ into the interior of its admissible interval, so each
strict inequality in which $x_i$ is highest-order holds with positive margin; a dependent
variable is set by its identity to the single admissible value $\ell_i=u_i$, meeting its two
non-strict encoding inequalities with margin $0$; and a free variable is highest-order in no
constraint, so it imposes nothing. Induction over the order gives
$\varphi(z)\models\Pi$.\hfill$\square$

\section{The Resolution Ratio as Condition Number}\label{app:theory}
Figure~\ref{fig:regimes} gives the intuition behind this appendix. The two regimes differ not in
the constraint but in the size of its margin relative to the data scale, and that ratio is what
makes the clamp and the predicate succeed or fail. The statements below make this precise.

\paragraph{A worked example.} Two of our orderings sit at the two ends of the range. Alchemy's $U-U_0$ has $R\approx9\times10^{4}$. The margin's spread $\sigma_m$ is about $10^5$ times smaller than the scale $\sigma_s$, so after standardization the margin occupies a band of relative width ${\sim}1/R\approx10^{-5}$, far below what a Lipschitz discriminator can resolve. A free generator that matches every feature's marginal therefore still violates the orderings on a constant fraction of samples (about half on average; RQ4). Trained in the loop, the constraint layer moves every sample onto $U{=}U_0$, and the distribution of the constraint margin collapses to a point mass at $0$. Taxi's $\mathrm{total}-\mathrm{fare}$ has $R\approx3.6$. The margin is about a third of the scale, the discriminator resolves it, the free generator already satisfies the constraint on most samples, the clamp rarely acts, and clamping leaves the distribution largely intact. The hybrid charts the former (high $R$) and clamps the latter (low $R$).

\begin{figure}[!htb]\centering
\includegraphics[width=\linewidth]{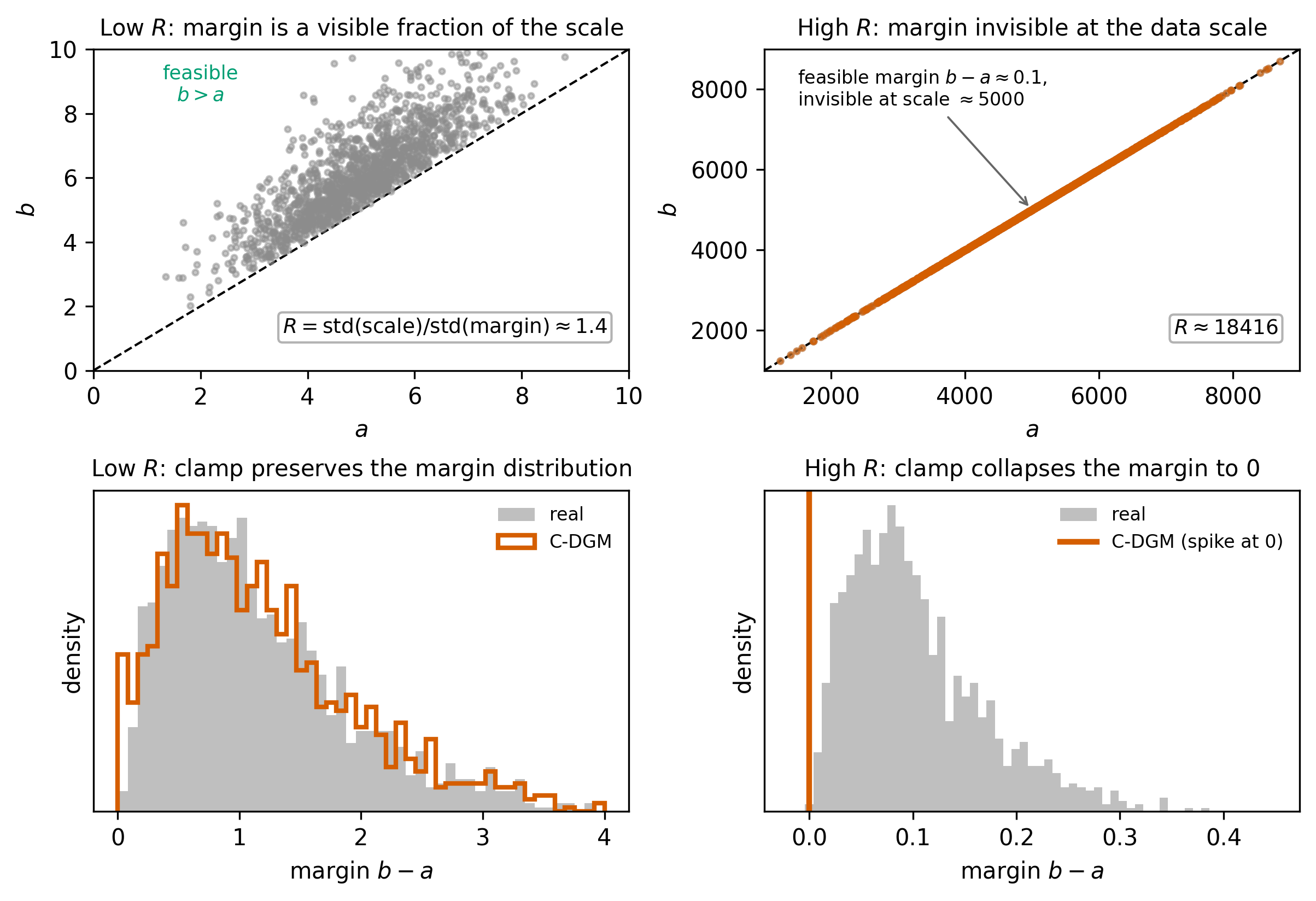}
\caption{\textbf{High versus low resolution ratio $R$ (schematic).} The same ordering $b>a$ in two
regimes. \emph{Top:} the data $(a,b)$ and the boundary $b=a$. At low $R$ (left) the feasible margin
$b-a$ is a visible fraction of the scale, so a free generator resolves it and the constraint
layer's clamp rarely acts. At high $R$ (right) the margin is tiny relative to the scale and
invisible at the data's magnitude, so every sample lies close to the boundary. \emph{Bottom:} the margin
distribution. At low $R$ the clamp (orange) preserves the distribution of the real constraint margin (grey); at high $R$ the
free generator cannot place the sub-resolution margin, so the clamp collapses every sample onto the
boundary, a point mass at $0$ disjoint from the real distribution. Function-symbol grounding repairs the
high-$R$ case by making the margin a unit-scale coordinate (Figure~\ref{fig:mechanism}).}
\label{fig:regimes}
\end{figure}

We now prove Corollaries~\ref{cor:identity} and~\ref{cor:ordering} and
Proposition~\ref{thm:collapse}.

\noindent\emph{Proof of Corollary~\ref{cor:identity}.} Full rank of $\nabla c$ on $\mathcal{M}_{=}$ makes $0$ a regular value, so by
the regular-value theorem $\mathcal{M}_{=}=c^{-1}(0)$ is an embedded $C^{1}$ submanifold of
dimension $D-p$; a submanifold of dimension below the ambient $D$ has Lebesgue measure zero. If
$\nu\ll\mathrm{Leb}$ with density $f$, then $\nu(\mathcal{M}_{=})=\int_{\mathcal{M}_{=}}f\,
\mathrm{d}\,\mathrm{Leb}=0$. Grounding sets the dependent coordinate to the exact value its identity
prescribes, so every generated $x$ lies in $\mathcal{M}_{=}$ by construction; the generated distribution is
then carried by the null set $\mathcal{M}_{=}$, so $\nu\not\ll\mathrm{Leb}$, escaping the
hypothesis, and $c(x)=0$ for every sample.\hfill$\square$

\noindent\emph{Proof of Corollary~\ref{cor:ordering}.} Differentiating, $\partial_{(b-a)}\sigma\!\big((b-a)/s\big)=
\sigma'\!\big((b-a)/s\big)/s=P_{s}(1-P_{s})/s$. The logistic derivative
$\sigma'(t)=\sigma(t)\,(1-\sigma(t))$ is $\Theta(1)$ on any fixed band $|t|\le c$ (peak $\sigma'(0)=\tfrac14$) and decays as $e^{-|t|}$ for
$|t|\gg1$, so the gradient is $\Theta(1/s)$ exactly for $|b-a|\lesssim s$ and exponentially
suppressed otherwise. With $b-a$ at scale $\sigma_{s}=R\,\sigma_{m}$ and density $\Theta(1/\sigma_{s})$
near $0$, the probability mass on the width-$\Theta(s)=\Theta(\sigma_{m})$ active band is
$\Theta(\sigma_{m}/\sigma_{s})=\Theta(1/R)$.\hfill$\square$

\medskip
\noindent\emph{Proof of Proposition~\ref{thm:collapse}.} \emph{(i)} On the ordering $\phi$, the
clamp leaves a satisfying sample's margin unchanged and moves each violator to the boundary
value $\varepsilon$, giving the mixture CDF
$F_{\mathrm{CL}}(t)=v\,\mathbf{1}[t\ge\varepsilon]+(1-v)\,F^{+}(t)$, with $F^{+}$ the CDF of
$\nu^{+}$. Two evaluations bound $\mathrm{KS}=\sup_t|F_{\mathrm{CL}}(t)-F(t)|$. At
$t=\varepsilon$, below the real support, $F_{\mathrm{CL}}(\varepsilon)-F(\varepsilon)\ge v$. At
$t=q_\alpha$, the real margin's $\alpha$-quantile (so $q_\alpha=\Theta(\sigma_m)$),
$F(q_\alpha)-F_{\mathrm{CL}}(q_\alpha)\ge \alpha-v-(1-v)\,F^{+}(q_\alpha)$, and under
Corollary~\ref{cor:ordering}'s density hypothesis the generator places mass
$\Theta\bigl(q_\alpha/(R\,\sigma_m)\bigr)=\Theta(1/R)$ in $(0,q_\alpha]$, so
$(1-v)\,F^{+}(q_\alpha)=\Theta(1/R)$. Taking
$\alpha\to1$,
$\mathrm{KS}\ge\max\bigl(v,\,1-v-\Theta(1/R)\bigr)\ge\max(v,1-v)-\Theta(1/R)\ge\tfrac12-\Theta(1/R)$,
and the first evaluation alone gives $\mathrm{KS}\to1$ as $v\to1$. With several constraints the
statement reads per ordering through its own clamp step; empirically the atom sits at the
boundary on every dataset (Figure~\ref{fig:margins}).
\emph{(ii)} Each link is a smooth bijection onto its admissible interval with smooth inverse:
softplus maps $\mathbb{R}$ onto $(0,\infty)$ with inverse
$\mathrm{softplus}^{-1}(y)=\log(e^{y}-1)$; the affine logistic map
$\ell+(u-\ell)\,\sigma(\cdot)$ maps $\mathbb{R}$ onto $(\ell,u)$ with inverse the scaled logit
$z=\log\frac{t}{1-t}$ at $t=(x-\ell)/(u-\ell)$; the heavy-tailed $\exp$ link maps $\mathbb{R}$ onto
$(\ell,\infty)$ with inverse $\log(x-\ell)$; a free coordinate is the identity. Assembling in
the Fourier-Motzkin order, the bounds of $x_i$ read only $x_{<i}$ (Appendix~\ref{app:proof}),
so the inverse is computed coordinate-wise. Each $z_i$ is the link inverse of $x_i$ at the
bounds set by $x_{<i}$, defined and smooth exactly when every constrained $x_i$ is strictly
interior, i.e.\ on the relative interior of $\mathcal{M}$ (dependent coordinates are recovered by their
identities and contribute no coordinate). Induction along the order gives
$\varphi\circ\varphi^{-1}=\mathrm{id}$ on $\mathrm{relint}\,\mathcal{M}$ and
$\varphi^{-1}\circ\varphi=\mathrm{id}$ on $\mathbb{R}^{d}$, and interior placement
(Appendix~\ref{app:proof}) means $\varphi$ places no probability on the boundary. Realizability: given any
margin distribution $\mu$ on $(0,\infty)$, the coordinate distribution $\mathrm{softplus}^{-1}_{\#}\mu$ pushes
forward under softplus to exactly $\mu$. When several reduced constraints bound the same
variable, $\ell_i$ and $u_i$ are maxima and minima of linear forms, hence piecewise-linear.
$\varphi$ and $\varphi^{-1}$ are then smooth off the measure-zero set where the active bound
switches and remain continuous bijections, so the validity, boundary, and realizability
claims are unaffected.\hfill$\square$

Figure~\ref{fig:mechanism} shows the mechanism of Remark~\ref{rem:cond} in pictures. The change of coordinates that
gives $\Theta(1)$ conditioning turns the sub-resolution margin of Figure~\ref{fig:regimes} into a
unit-scale coordinate the GAN can learn, then decodes it back to valid samples with the distribution of the constraint margin
recovered.

\begin{figure}[!htb]\centering
\includegraphics[width=0.88\linewidth]{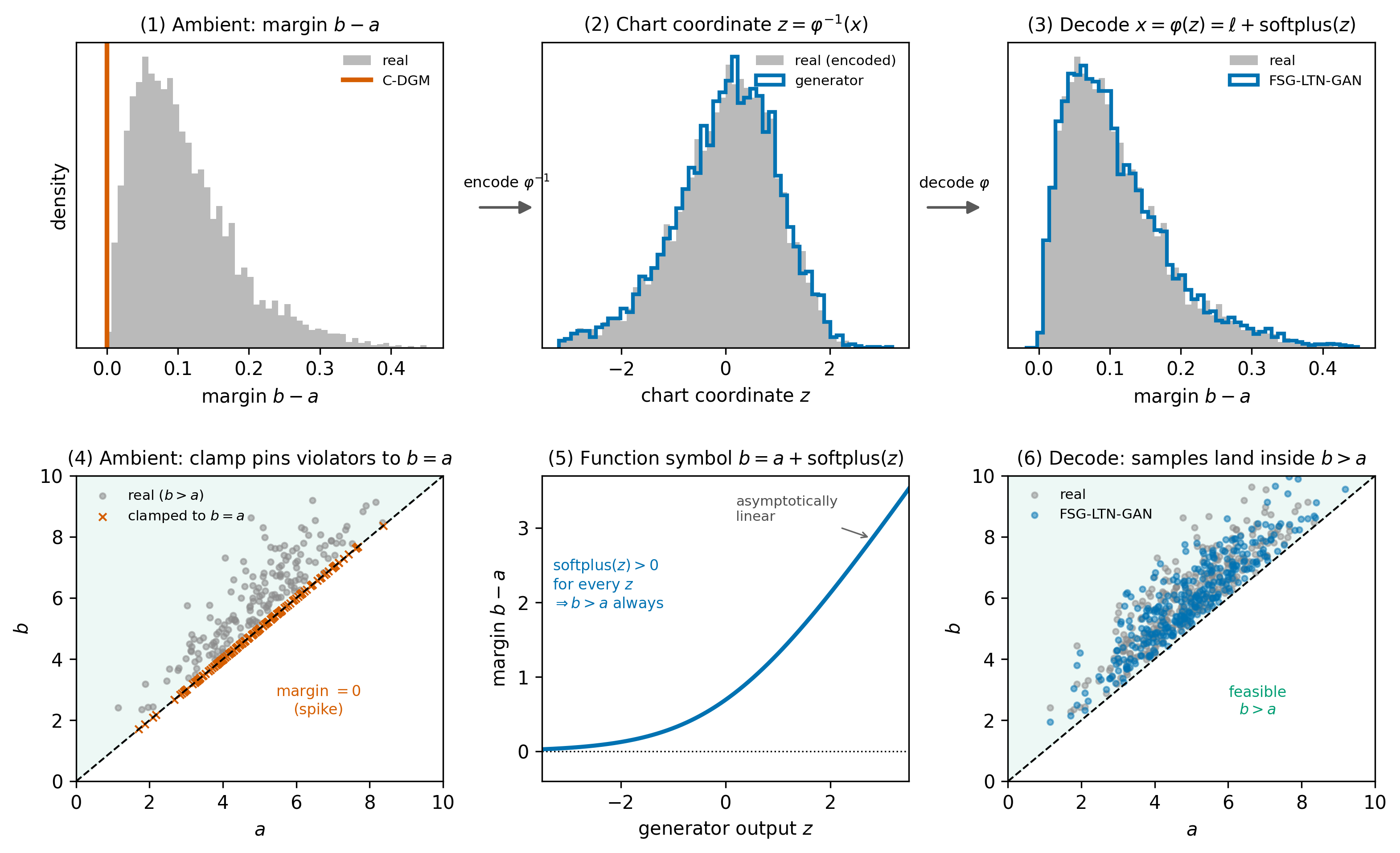}
\caption{\textbf{How function-symbol grounding repairs the high-$R$ collapse (schematic).} The same
high-$R$ ordering margin as in Figure~\ref{fig:regimes}; top row shows the margin
\emph{distribution} at each stage, bottom row the corresponding $(a,b)$ \emph{geometry}.
\emph{(1)} In ambient space the margin $b-a$ is sub-resolution, so the constraint layer clamps it
to a point mass at $0$. \emph{(2)} The chart coordinate $z=\varphi^{-1}(x)$ standardizes the margin to
unit scale, where the discriminator resolves it and a standard GAN learns the distribution of the real constraint margin (generator
in blue matches real in grey). \emph{(3)} Decoding $x=\varphi(z)=\ell+\mathrm{softplus}(z)$
carries that distribution back to the ambient margin. \emph{(4)} Geometrically the clamp moves every violating
sample onto the boundary $b=a$ (margin $0$, the point mass of panel~1). \emph{(5)} The function symbol
$b=a+\mathrm{softplus}(z)$ is positive for every generator output $z$, so $b>a$ holds by
construction ($\mathrm{Sat(KB)}=1$); it is smooth and monotone, so matching the distribution of
$z$ matches the margin's distribution, and asymptotically linear, so generator tails do not diverge.
\emph{(6)} Our decoded samples therefore land inside the feasible region $b>a$, matching the real
distribution with exact validity.}
\label{fig:mechanism}
\end{figure}

\section{Constraint-Layer Code Corrections and Evaluation Protocol}\label{app:fair}
We run the constraint layer from the code released by \citet{stoian2024cdgm}. For genuine $100\%$ validity (a fair
margin comparison) we correct a latent bug in its Fourier-Motzkin reduction, which keeps a stale term for a
variable whose coefficients cancel (triggered by equalities encoded as two inequalities), and run the layer in double
precision with the strict-inequality $\varepsilon$ set above one unit in the last place (its default
$\varepsilon=10^{-12}$ underflows here, falsely reporting $0\%$ valid). Our decoder runs in
double precision on Taxi, nycflights13, and the RQ3 benchmark, and in single precision on the
chemistry sets, whose reference data are single precision. These changes only help the constraint layer, and the distortion it
induces is unchanged. All generators are sampled in \texttt{eval} mode. The metrics (KS and satisfaction fractions) cannot be
improved by shrinking magnitudes. Margins are scored
against a fixed reference sample of the real data. On the four high-resolution datasets and
nycflights13, which ship as single files with no canonical split, this is a $20{,}000$-row
sample drawn once with a fixed seed from the same data the generators train on. On the RQ3
benchmark it is the held-out validation and test splits. Every method in a comparison is
scored against the same reference. The grounding link (softplus or $\exp$) is set by the fixed
dynamic-range rule of Appendix~\ref{app:hparams}. A definitional identity is scored satisfied when its residual is
within a per-dataset tolerance in raw units ($0.05$ on Alchemy, $10^{-3}$ on tmQM, $0.5$ on
nycflights13). Corollary~\ref{cor:identity} concerns exact satisfaction. The tolerance is why
unconstrained validity is small but nonzero on the chemistry sets, while the wider-scaled
nycflights13 identity is almost never met and its validity is $0.000$. Per-dataset $R$,
sizes, constraints, and full per-method results with standard deviations are given in
Appendices~\ref{app:data} to~\ref{app:extra}.

\section{Datasets, Constraints, and Resolution Ratios}\label{app:data}
Table~\ref{tab:dataspec} summarizes the four real high-resolution datasets used in RQ1, RQ2, and
RQ4. Each property is standardized to zero mean and unit variance before training. The resolution
ratio $R=\sigma_s/\sigma_m$ (Section~\ref{sec:problem}) is computed on the raw data
before training. The constraint-layer benchmark of \citet{stoian2024cdgm} used in RQ3 is
described in Appendix~\ref{app:rq3full}.

\begin{table}[H]\centering
\caption{The four real high-resolution datasets. $N$ rows, $D$ properties; $R$ is the resolution
ratio of each evaluated ordering margin.}
\label{tab:dataspec}
\small
\begin{tabular}{lccll}
\toprule
Dataset & $N$ & $D$ & Evaluated margin(s) & $R$ \\
\midrule
Alchemy & 202{,}579 & 12 & $U-U_0$;\ $H-U$;\ $H-G$ & $9{\times}10^{4}$;\ $7{\times}10^{6}$;\ $3{\times}10^{4}$ \\
tmQM & 108{,}541 & 8 & $E_{\mathrm{elec}}-E_{\mathrm{tot}}$ & $2.3{\times}10^{4}$ \\
Transition1x & 10{,}073 & 3 & $E_{\mathrm{TS}}-E_{\mathrm{react}}$;\ $E_{\mathrm{TS}}-E_{\mathrm{prod}}$ & $1.0{\times}10^{3}$;\ $8.1{\times}10^{2}$ \\
Taxi & 50{,}000 & 6 & duration;\ $\mathrm{total}-\mathrm{fare}$ & $1.1{\times}10^{3}$;\ $3.6$ \\
\bottomrule
\end{tabular}
\end{table}

\noindent The full structural constraint set per dataset is as follows.
\textbf{Alchemy} \citep{chen2019alchemy}, 12 molecular properties (mu, alpha, homo, lumo, gap, r2,
zpve, U0, U, H, G, Cv, following QM9 naming): positivity of mu (dipole moment), gap, and zpve
(zero-point vibrational energy); the algebraic identity
$\mathrm{lumo}=\mathrm{homo}+\mathrm{gap}$ over the frontier-orbital energies; the
thermochemical orderings $U_0<U<H$ and $G<H$, where $U_0$ and $U$ are the internal energies at
$0$\,K and $298$\,K, $H$ the enthalpy, and $G$ the free energy.
\textbf{tmQM} \citep{balcells2020tmqm}, 8 properties of transition-metal complexes: the high-$R$
dispersion ordering $E_{\mathrm{tot}}<E_{\mathrm{elec}}$; positivity of dipole, gap, and
polarizability; the low-$R$ identity $\mathrm{LUMO}=\mathrm{HOMO}+\mathrm{gap}$.
\textbf{Transition1x} \citep{schreiner2022transition1x}, reaction energy profiles: the transition
state is the maximum, $E_{\mathrm{TS}}>E_{\mathrm{react}}$ and $E_{\mathrm{TS}}>E_{\mathrm{prod}}$.
($0.08\%$ of raw rows violate $E_{\mathrm{TS}}>E_{\mathrm{prod}}$, near-barrierless reverse
reactions, and are retained, so the real reference margin carries that small negative mass.)
\textbf{Taxi} \citep{nyctlc}, NYC yellow-cab records (2024-01): the duration ordering
$\mathrm{dropoff}>\mathrm{pickup}$ on absolute second-scale timestamps (high $R$);
$\mathrm{total}>\mathrm{fare}$ (low $R$); positivity of distance and fare.

\section{Architecture and Training Protocol}\label{app:hparams}
All methods share one generator architecture so that only the constraint mechanism differs. The generator is a
multilayer perceptron with latent dimension 64 and three hidden layers of width 256 with
LeakyReLU($0.2$) and dropout $0.1$, and batch normalization on the two inner layers, followed by a
linear map to the property dimension. Weights use Kaiming-uniform initialization. The
discriminator is a multilayer perceptron $D\to256\to128\to1$ with LeakyReLU($0.2$), dropout $0.1$,
and a sigmoid output. Both train with Adam (learning rate $2{\times}10^{-4}$, $\beta=(0.5,0.999)$),
batch size 256, binary cross-entropy adversarial loss with label smoothing (targets $0.9$ real,
$0.1$ generated), for 1000 steps. Two numerical safeguards apply throughout. The decoder floors
each charted margin at $10^{-6}$ in raw units, and the encoder clips the real chart coordinates
at their $0.5$ and $99.5$ percentiles before standardization. Both act only in a thin boundary
layer, and neither affects validity. The analysis of Section~\ref{sec:method} concerns the
exact map. All experiments use ten
seeds (0 to 9, except the Alchemy and Transition1x inverse-design runs, which use 10 to 19).
For function-symbol grounding the same generator emits the
free coordinates $z$ and the grounding map $\varphi$ assembles the sample
(Section~\ref{sec:method}). No satisfaction loss is added because charted axioms hold by
construction. The constraint layer runs on the identical architecture
(Appendix~\ref{app:fair}). The grounding link is softplus by default and $\exp$ (a
multiplicative increment) for heavy-tailed margins. This is the method's only added
hyperparameter beyond the hybrid's two thresholds, and
it was set once per dataset by a fixed criterion. The criterion measures each charted margin's
dynamic range as the ratio $q_{99}/q_{50}$ of its positive part on the training data, and selects
the $\exp$ link when any charted margin reaches $q_{99}/q_{50}\ge 10$. Across the eleven datasets in
this paper it selects $\exp$ exactly once, on faults, whose bounding-box margins have
$q_{99}/q_{50}=18.0$ and $13.0$. It selects softplus everywhere else (every charted margin outside
faults has $q_{99}/q_{50}\le 5.1$). The released code records the resulting per-dataset choice.

\section{Full Per-Method Results (RQ1 and RQ2)}\label{app:fullrq}
Table~\ref{tab:rq1full} extends Table~\ref{tab:rq1} with all baselines and their validities;
Table~\ref{tab:rq2full} extends Table~\ref{tab:rq2} with the raw unconstrained architectures,
whose validities are given in its caption. The paired Wilcoxon test for FSG-LTN-GAN versus C-DGM returns $p=0.00195$ on
every dataset, the minimum attainable at $n=10$. In Table~\ref{tab:rq1full}, the projection row
coincides with the unconstrained row on tmQM for the reason given in
Table~\ref{tab:rq2full}'s caption. Moving violators to the boundary leaves the KS supremum
unchanged when the reference margin lies above it.

\begin{table}[H]\centering
\caption{Full RQ1 results, $n=10$ seeds, mean $\pm$ std. Validity is the fraction satisfying all
constraints. Margin KS is the two-sample KS between the generated margins and a fixed real reference sample (lower is
better). C-DGM and FSG-LTN-GAN are both exactly valid. Only FSG-LTN-GAN recovers the distribution of the constraint margin.}
\label{tab:rq1full}
\small
\begin{tabular}{llcc}
\toprule
Dataset & Method & Validity & Margin KS $\downarrow$ \\
\midrule
Alchemy & unconstrained GAN & $0.053\pm0.021$ & $0.539\pm0.014$ \\
 & projection & $1.000\pm0.000$ & $0.664\pm0.026$ \\
 & C-DGM (clamp) & $1.000\pm0.000$ & $1.000\pm0.000$ \\
 & \textbf{FSG-LTN-GAN} & $1.000\pm0.000$ & $\mathbf{0.040\pm0.007}$ \\
\midrule
tmQM & unconstrained GAN & $0.229\pm0.046$ & $0.575\pm0.037$ \\
 & projection & $1.000\pm0.000$ & $0.575\pm0.037$ \\
 & C-DGM (clamp) & $1.000\pm0.000$ & $1.000\pm0.000$ \\
 & \textbf{FSG-LTN-GAN} & $1.000\pm0.000$ & $\mathbf{0.045\pm0.010}$ \\
\midrule
Transition1x & unconstrained GAN & $0.293\pm0.384$ & $0.833\pm0.140$ \\
 & projection & $1.000\pm0.000$ & $0.881\pm0.093$ \\
 & C-DGM (clamp) & $1.000\pm0.000$ & $0.999\pm0.000$ \\
 & \textbf{FSG-LTN-GAN} & $1.000\pm0.000$ & $\mathbf{0.065\pm0.010}$ \\
\midrule
Taxi & unconstrained GAN & $0.253\pm0.087$ & $0.346\pm0.042$ \\
 & projection & $0.483\pm0.112$ & $0.345\pm0.042$ \\
 & C-DGM (clamp) & $1.000\pm0.000$ & $0.574\pm0.005$ \\
 & \textbf{FSG-LTN-GAN} & $1.000\pm0.000$ & $\mathbf{0.099\pm0.010}$ \\
\bottomrule
\end{tabular}
\end{table}

\begin{table}[!htb]\centering
\caption{Full RQ2 (margin KS by architecture, $n=10$, mean $\pm$ std). Raw CTGAN/TVAE validity is
below $3\%$ on the chemistry sets and $30$ to $46\%$ on Transition1x and Taxi. The $+$CL rows,
C-DGM, and FSG-LTN-GAN all reach $100\%$ validity. Only FSG-LTN-GAN recovers the distribution of the constraint margin across the architectures we
test. On tmQM and Taxi the raw and $+$CL margin-KS entries coincide exactly. This is forced
rather than copied, because the clamp maps the violating mass to the boundary point mass at
$\varepsilon$ while the reference margin lies above $\varepsilon$, so the KS supremum, attained
at $\varepsilon$, has the same value before and after. Validity and moment error do move. On Alchemy and
Transition1x the chained orderings share endpoints, so clamping shifts the evaluated margins and
the KS.}
\label{tab:rq2full}
\footnotesize
\setlength{\tabcolsep}{4pt}
\begin{tabular}{lcccc}
\toprule
Method & Alchemy & tmQM & Transition1x & Taxi \\
\midrule
CTGAN (raw) & $0.596\pm0.029$ & $0.579\pm0.052$ & $0.569\pm0.057$ & $0.419\pm0.024$ \\
CTGAN + CL & $0.674\pm0.036$ & $0.579\pm0.052$ & $0.592\pm0.040$ & $0.419\pm0.024$ \\
TVAE (raw) & $0.528\pm0.012$ & $0.517\pm0.011$ & $0.478\pm0.047$ & $0.348\pm0.029$ \\
TVAE + CL & $0.651\pm0.015$ & $0.517\pm0.011$ & $0.510\pm0.053$ & $0.348\pm0.029$ \\
C-DGM (CL, MLP) & $1.000\pm0.000$ & $1.000\pm0.000$ & $0.999\pm0.000$ & $0.574\pm0.005$ \\
\textbf{FSG-LTN-GAN} & $\mathbf{0.040\pm0.007}$ & $\mathbf{0.045\pm0.010}$ & $\mathbf{0.065\pm0.010}$ & $\mathbf{0.099\pm0.010}$ \\
\bottomrule
\end{tabular}
\end{table}

\section{The Constraint-Layer Benchmark, Full (RQ3)}\label{app:rq3full}
Table~\ref{tab:rq3full} extends Table~\ref{tab:rq3} with each dataset's constraint count and
charted count. KS is averaged over each dataset's constraints. C-DGM, FSG-all, and the hybrid all
reach $100\%$ validity; unconstrained validity ranges from $0.005$ (wids) to $0.736$ (url). The
hybrid charts the constraints its pre-run selects (free-generator satisfaction below $0.9$, margin
not discrete) and clamps the rest. On heloc and lcld it charts nothing and reduces to the
constraint layer. On url, news, and wids it charts one or two variables, with smaller but
significant gains on url and wids ($p=0.03$ and $0.004$) and a statistical tie on news
(heloc and lcld are ties as well). On faults (whose high-$R$ margin is the bounding-box
ordering $Y_{\max}>Y_{\min}$) it charts every variable and improves substantially (paired
Wilcoxon $p=0.002$, all ten seeds). The FSG-all
ablation, which charts every variable, is worse than the constraint layer on every dataset but
faults.

\begin{table}[H]\centering
\caption{Full RQ3 results on the benchmark of \citet{stoian2024cdgm}, $n=10$ seeds, mean $\pm$ std,
margin KS. ``charted'' is the number of variables the hybrid charts as function symbols;
\textbf{bold} marks where the hybrid significantly outperforms the constraint layer (paired Wilcoxon
$p<0.05$).}
\label{tab:rq3full}
\small
\begin{tabular}{lcccccc}
\toprule
Dataset & $|\Pi|$ & charted & $R_{\max}$ & C-DGM & FSG-all & \textbf{hybrid} \\
\midrule
faults & 4 & 4 & $1.8{\times}10^{3}$ & $0.745\pm0.014$ & $0.117\pm0.008$ & $\mathbf{0.117\pm0.008}$ \\
heloc & 7 & 0 & $2$ & $0.156\pm0.012$ & $0.208\pm0.005$ & $0.158\pm0.014$ \\
lcld & 4 & 0 & $1$ & $0.131\pm0.008$ & $0.657\pm0.004$ & $0.133\pm0.007$ \\
url & 8 & 1 & $1$ & $0.159\pm0.010$ & $0.196\pm0.018$ & $\mathbf{0.144\pm0.021}$ \\
news & 5 & 2 & $2$ & $0.294\pm0.008$ & $0.318\pm0.007$ & $0.297\pm0.007$ \\
wids & 31 & 1 & $4$ & $0.140\pm0.007$ & $0.319\pm0.004$ & $\mathbf{0.133\pm0.008}$ \\
\bottomrule
\end{tabular}
\end{table}

\section{Additional Analyses}\label{app:extra}
\paragraph{Density and coverage.} Table~\ref{tab:denscov} reports density and
coverage \citep{naeem2020reliable}, the constraint layer's own realism metrics,
computed on standardized features and shown alongside margin KS.

\begin{table}[H]\centering
\caption{Density and coverage \citep{naeem2020reliable} ($n=10$, mean $\pm$ std), with margin
KS for reference. Higher is better for density and coverage, lower for KS. The better method
per column is in bold.}
\label{tab:denscov}
\small
\begin{tabular}{llccc}
\toprule
Dataset & Method & Density & Coverage & Margin KS $\downarrow$ \\
\midrule
Alchemy & C-DGM & $\mathbf{0.923\pm0.031}$ & $0.619\pm0.014$ & $1.000\pm0.000$ \\
 & \textbf{FSG-LTN-GAN} &$0.780\pm0.034$ & $\mathbf{0.636\pm0.017}$ & $\mathbf{0.040\pm0.007}$ \\
\midrule
tmQM & C-DGM & $\mathbf{0.839\pm0.015}$ & $\mathbf{0.752\pm0.009}$ & $1.000\pm0.000$ \\
 & \textbf{FSG-LTN-GAN} &$0.801\pm0.022$ & $0.735\pm0.011$ & $\mathbf{0.045\pm0.010}$ \\
\midrule
Transition1x & C-DGM & $0.200\pm0.039$ & $0.050\pm0.005$ & $0.999\pm0.000$ \\
 & \textbf{FSG-LTN-GAN} &$\mathbf{0.248\pm0.030}$ & $\mathbf{0.203\pm0.009}$ & $\mathbf{0.065\pm0.010}$ \\
\midrule
Taxi & C-DGM & $0.066\pm0.008$ & $0.052\pm0.004$ & $0.574\pm0.005$ \\
 & \textbf{FSG-LTN-GAN} &$\mathbf{0.870\pm0.013}$ & $\mathbf{0.704\pm0.015}$ & $\mathbf{0.099\pm0.010}$ \\
\bottomrule
\end{tabular}
\end{table}

\paragraph{The failure is not a variable-ordering artifact.} A constraint layer must fix a
Fourier-Motzkin elimination order. Table~\ref{tab:ordering} runs the constraint layer on
Transition1x under every dependency-valid ordering: all stay near $1.0$ margin KS, so this
failure is intrinsic to clamping. Function-symbol grounding,
which fixes one ordering, is at $0.065$. All configurations are $100\%$ valid. (The two
apex-first orders reduce to the same triangular program and give identical runs.)

\begin{table}[H]\centering
\caption{Constraint-layer robustness to variable ordering (Transition1x, $n=10$, mean $\pm$ std,
margin KS). Bracketed lists are the Fourier-Motzkin variable elimination orders.}
\label{tab:ordering}
\small
\begin{tabular}{lc}
\toprule
Configuration & Margin KS $\downarrow$ \\
\midrule
C-DGM, natural order $[0,1,2]$ & $0.999\pm0.000$ \\
C-DGM, apex-first $[1,0,2]$ & $0.999\pm0.001$ \\
C-DGM, apex-first $[1,2,0]$ & $0.999\pm0.001$ \\
C-DGM, apex-last $[0,2,1]$ & $0.903\pm0.049$ \\
\textbf{FSG-LTN-GAN} & $\mathbf{0.065\pm0.010}$ \\
\bottomrule
\end{tabular}
\end{table}

\paragraph{Other backbones: CTGAN and TVAE in the chart.} Table~\ref{tab:rq2} showed
the \emph{failure} on CTGAN and TVAE. The chart supplies the \emph{fix} for them as well.
Because function-symbol grounding is a change of coordinates on the data, any tabular generator
can be trained in the chart: encode the training data by $\varphi^{-1}$, fit the unmodified
backbone with its package-default hyperparameters at $100$ epochs (the RQ2 protocol), and
decode its samples through $\varphi$.
Table~\ref{tab:backbones} applies this recipe to CTGAN and TVAE with no architectural change:
both become exactly valid and recover the margin, with margin KS $3.6$ to $8.6\times$ below their
constraint-layer counterparts (paired Wilcoxon $p{=}0.002$ on every dataset--backbone pair). One
backbone-specific effect persists. TVAE's latent-variance shrinkage, applied in chart coordinates,
propagates through the free base variable and inflates ambient per-property moment error on
tmQM and Transition1x (CTGAN in the chart does not show this, and margin KS is
unaffected either way).

\begin{table}[H]\centering
\caption{\textbf{The method transfers to other generator families} (margin KS $\downarrow$,
$n{=}10$, mean $\pm$ std). ``in chart'' trains the unmodified backbone on
$\varphi^{-1}$-encoded data and decodes through $\varphi$. Both chart columns are exactly
$100\%$ valid. ``$+$CL'' columns from Table~\ref{tab:rq2}.}
\label{tab:backbones}
\small
\setlength{\tabcolsep}{4.5pt}
\begin{tabular}{lcccc}
\toprule
Dataset & CTGAN$+$CL & \textbf{CTGAN in chart} & TVAE$+$CL & \textbf{TVAE in chart} \\
\midrule
Alchemy      & $0.674\pm0.036$ & $\mathbf{0.105\pm0.041}$ & $0.651\pm0.015$ & $\mathbf{0.127\pm0.005}$ \\
tmQM         & $0.579\pm0.052$ & $\mathbf{0.109\pm0.061}$ & $0.517\pm0.011$ & $\mathbf{0.060\pm0.014}$ \\
Transition1x & $0.592\pm0.040$ & $\mathbf{0.089\pm0.019}$ & $0.510\pm0.053$ & $\mathbf{0.141\pm0.037}$ \\
Taxi         & $0.419\pm0.024$ & $\mathbf{0.086\pm0.016}$ & $0.348\pm0.029$ & $\mathbf{0.064\pm0.014}$ \\
\bottomrule
\end{tabular}
\end{table}

\paragraph{Flight records (nycflights13).} The four main datasets and the RQ3 benchmark
were each chosen by the authors of one of the two papers being compared. We therefore ran a
fifth dataset that neither paper had used, the nycflights13 flight records ($327{,}346$ rows after cleaning, subsampled to $50{,}000$; six
properties), with the ordering $\mathrm{arrival}>\mathrm{departure}$ on absolute timestamps
spanning a year, the exact identity $\mathrm{dep}=\mathrm{sched}+\mathrm{delay}$, and positivity
of air time and distance. \emph{Before training} we computed $R=3.0\times10^{3}$ for the
ordering from the raw data and recorded the predicted outcome (constraint-layer margin KS near
$1$, FSG-LTN-GAN below $0.15$, both exactly valid; unconstrained validity near $0$), together
with explicit falsification criteria, in a file included in the code
release. The sweep then ran once, with ten seeds. Every prediction held
(Table~\ref{tab:flights}). The clamp collapses
the flight-duration margin to the boundary exactly as on Taxi and the chemistry sets, and the
chart recovers it, at equal exact validity.

\begin{table}[H]\centering
\caption{\textbf{Flight records (nycflights13)} ($n{=}10$, mean $\pm$ std). $R$
and the predicted outcome were recorded before training.}
\label{tab:flights}
\small
\begin{tabular}{lcc}
\toprule
Method & Validity & Margin KS $\downarrow$ \\
\midrule
unconstrained GAN & $0.000\pm0.000$ & $0.581\pm0.044$ \\
projection & $0.530\pm0.068$ & $0.541\pm0.041$ \\
C-DGM (clamp) & $1.000\pm0.000$ & $1.000\pm0.000$ \\
\textbf{FSG-LTN-GAN} & $1.000\pm0.000$ & $\mathbf{0.086\pm0.011}$ \\
\bottomrule
\end{tabular}
\end{table}

\paragraph{Where the margin-KS gain comes from.} FSG-LTN-GAN differs from the constraint layer
in two coupled ways: samples are valid by construction, and the discriminator operates on the
chart coordinates $z$, where every margin is unit scale, while the constraint layer's
discriminator sees the original, ill-scaled coordinates. To separate the two contributions we
train \emph{FSG-ambient-D}: the generator and the chart $\varphi$ are unchanged, so validity
remains exact, but the discriminator receives the decoded, standardized ambient sample instead
of $z$. Table~\ref{tab:ambientd} shows margin KS degrades on every dataset (paired Wilcoxon
$p{=}0.002$ each), and the degradation tracks $R$ (largest on the chemistry sets), yet remains
well below the clamp's. Both mechanisms therefore contribute. Generating \emph{inside} the
feasible region rather than clamping onto its boundary already improves the margin, and on
three of the four datasets the larger share of the recovery comes from the discriminator
operating at unit scale. On Transition1x the shares reverse, with generating inside
contributing most of the difference. The two are parts of the same
grounding. The chart supplies the coordinates, and the discriminator is best run in them. (Validity
is $1.000$ in every run except a single tmQM seed at $0.999$, a float32
identity-reconstruction round-off in the ambient arm's training-time decode.)

\begin{table}[H]\centering
\caption{\textbf{Validity by construction versus discriminating in the chart.} Margin KS
($n{=}10$, mean $\pm$ std). FSG-ambient-D keeps the chart generator (validity exact) but shows
the discriminator the decoded ambient sample; C-DGM from Table~\ref{tab:rq1full} for reference.}
\label{tab:ambientd}
\small
\begin{tabular}{lccc}
\toprule
Dataset & \textbf{FSG (D in chart)} & FSG-ambient-D & C-DGM (clamp) \\
\midrule
Alchemy      & $\mathbf{0.040\pm0.007}$ & $0.838\pm0.154$ & $1.000\pm0.000$ \\
tmQM         & $\mathbf{0.045\pm0.010}$ & $0.737\pm0.339$ & $1.000\pm0.000$ \\
Transition1x & $\mathbf{0.065\pm0.010}$ & $0.267\pm0.106$ & $0.999\pm0.000$ \\
Taxi         & $\mathbf{0.099\pm0.010}$ & $0.340\pm0.126$ & $0.574\pm0.005$ \\
\bottomrule
\end{tabular}
\end{table}

\paragraph{Margin-form preprocessing is a manual chart.} For a single ordering one can instead
transform the \emph{data}: on Taxi, replace dropoff with
$\mathrm{duration}=\mathrm{dropoff}-\mathrm{pickup}$, standardize, train the constraint layer
with $\mathrm{duration}>0$, and decode afterwards. Because the margin becomes its own column, its resolution
ratio drops to $R{=}1$ (from $1.1\times10^{3}$), and the clamp then largely preserves the margin distribution. Margin
KS falls $0.574\to0.134\pm0.020$ (duration KS $1.000\to0.118\pm0.033$) at exact validity
(Table~\ref{tab:marginform}). This is further evidence for the coordinate explanation. The preprocessing \emph{is} function-symbol grounding applied to the data by
hand, for one constraint. Three observations make the chart its general form. First, the
transform alone confers no validity. The same preprocessing with an unconstrained GAN is
$52.1\%\pm5.5$ valid (duration can still be generated negative, and the untouched constraints
are violated), so a clamp or a chart is still required. Second, rewriting a \emph{conjunction}
into consistent margin form (Alchemy's $U_0<U<H$ chain together with its lumo identity) must
proceed variable by variable in a triangular order, which is exactly the Fourier-Motzkin
substitution the chart automates. Third, the chart also standardizes what the discriminator sees
for \emph{all} constraints at once (previous paragraph), which manual rewriting achieves only
for the rewritten margins. FSG-LTN-GAN accordingly remains best ($0.099\pm0.010$; paired
Wilcoxon $p{=}0.002$ against the preprocessed constraint layer, all ten seeds).

\begin{table}[!htb]\centering
\caption{\textbf{Margin-form preprocessing on Taxi} ($n{=}10$, mean $\pm$ std). Preprocessing
the data into margin form is a manual, single-constraint chart: it recovers the clamp's margin,
though not validity (without a constraint mechanism) and not the untransformed margins.
Duration KS is the KS of the high-$R$ duration margin alone.}
\label{tab:marginform}
\small
\begin{tabular}{lccc}
\toprule
Method & Validity & Margin KS $\downarrow$ & Duration KS $\downarrow$ \\
\midrule
C-DGM, ambient (Table~\ref{tab:rq1full}) & $1.000\pm0.000$ & $0.574\pm0.005$ & $1.000\pm0.000$ \\
margin-form $+$ C-DGM & $1.000\pm0.000$ & $0.134\pm0.020$ & $0.118\pm0.033$ \\
margin-form $+$ unconstrained GAN & $0.521\pm0.055$ & $0.121\pm0.015$ & $0.111\pm0.024$ \\
\textbf{FSG-LTN-GAN} & $1.000\pm0.000$ & $\mathbf{0.099\pm0.010}$ & $\mathbf{0.081\pm0.024}$ \\
\bottomrule
\end{tabular}
\end{table}

\paragraph{Conditional inverse design.} The margin KS matters
when a consumer of the samples needs the constrained quantity to be realistic. In
\emph{conditional inverse design}, we condition the
generator on a target value $t$ for a designable property (the HOMO--LUMO gap on Alchemy,
the reactant energy on Transition1x, the electronic energy on tmQM) and ask for complete property
profiles that are (i) valid under the dataset's constraints, (ii) on target, and (iii) realistic
in their margins. All arms share the conditional architecture (the discriminator sees the sample--target pair). They differ only in the constraint mechanism: none, a predicate-style
satisfaction penalty, post-hoc projection, or the chart. In Table~\ref{tab:invdesign},
function-symbol grounding is the only configuration that delivers all three at once: exact
validity (the tmQM chart value is $0.9999$ before rounding, a float32 identity round-off), target error at the
unconstrained level, and margin KS an order of magnitude below
every alternative. The penalty arm shows the cost of soft satisfaction at high $R$. Pushed toward
satisfaction, it misses the target (target error $2$ to $5\times$ worse) yet still fails
validity. Projection attains validity but inherits the unconstrained margins.

\begin{table}[!htb]\centering
\caption{\textbf{Conditional inverse design} ($n{=}10$ seeds, mean $\pm$ std, 1500 steps).
Target error is the mean absolute deviation of the designable property from its conditioning
target (standardized units). Margin KS is as in the main text. Only the chart is simultaneously
valid, on target, and realistic in its margins.}
\label{tab:invdesign}
\small
\setlength{\tabcolsep}{4.5pt}
\begin{tabular}{llccc}
\toprule
Dataset & Mechanism & Validity & Target err.\ $\downarrow$ & Margin KS $\downarrow$ \\
\midrule
Alchemy & unconstrained & $0.089\pm0.050$ & $0.067\pm0.003$ & $0.629\pm0.072$ \\
(target: gap) & penalty & $0.019\pm0.009$ & $0.349\pm0.091$ & $0.624\pm0.035$ \\
 & projection & $0.999\pm0.001$ & $0.067\pm0.003$ & $0.652\pm0.050$ \\
 & \textbf{chart (FSG)} & $\mathbf{1.000\pm0.000}$ & $0.073\pm0.004$ & $\mathbf{0.049\pm0.008}$ \\
\midrule
Transition1x & unconstrained & $0.662\pm0.235$ & $0.053\pm0.004$ & $0.477\pm0.170$ \\
(target: $E_{\mathrm{react}}$) & penalty & $0.797\pm0.043$ & $0.124\pm0.009$ & $0.851\pm0.031$ \\
 & projection & $1.000\pm0.000$ & $0.053\pm0.004$ & $0.470\pm0.174$ \\
 & \textbf{chart (FSG)} & $\mathbf{1.000\pm0.000}$ & $0.060\pm0.006$ & $\mathbf{0.049\pm0.010}$ \\
\midrule
tmQM & unconstrained & $0.269\pm0.061$ & $0.065\pm0.004$ & $0.548\pm0.042$ \\
(target: $E_{\mathrm{elec}}$) & penalty & $0.036\pm0.016$ & $0.253\pm0.020$ & $0.608\pm0.051$ \\
 & projection & $1.000\pm0.000$ & $0.065\pm0.004$ & $0.548\pm0.042$ \\
 & \textbf{chart (FSG)} & $\mathbf{1.000\pm0.000}$ & $0.073\pm0.008$ & $\mathbf{0.038\pm0.006}$ \\
\bottomrule
\end{tabular}
\end{table}

\paragraph{Sensitivity of the hybrid's thresholds.} The hybrid's two thresholds are fixed once
($\tau_s{=}0.9$, $\tau_d{=}0.2$) and shared by every dataset in the paper. Sweeping the grid
$\tau_s\in\{0.70,0.80,0.85,0.90,0.95\}\times\tau_d\in\{0.05,0.10,0.20,0.30,0.50\}$ changes the
charted set on $5$ to $19$ of the $25$ cells, depending on the dataset, and retraining every
distinct alternative set the grid produces (ten seeds each, Table~\ref{tab:threshsens}) changes
no conclusion: every configuration keeps exact validity, alternates move margin KS by at most
${\approx}0.03$ on the low-$R$ datasets, and every cell near the paper's setting keeps every
win and tie. The one qualification is wids, where single extreme cells that chart seven or
more constraints move to KS $0.143$ to $0.152$, at or slightly above the constraint layer's
$0.140$. The one large
change shows where the sensitivity lies. On faults, the extreme $\tau_d{=}0.05$ row stops charting the two
bounding-box margins (including the high-$R$ one), and the main win shrinks (KS $0.117\to0.548$, still below the
constraint layer's $0.745$). The outcome is sensitive not to the threshold values but to
whether the high-$R$ margins are charted, which is what $R$ predicts before training.

\begin{table}[H]\centering
\caption{\textbf{Threshold sensitivity on the RQ3 benchmark} (margin KS, $n{=}10$, mean $\pm$
std). ``cells'' is how many of the $25$ grid cells select each charted set. The paper's cell is
$(\tau_s,\tau_d){=}(0.9,0.2)$.}
\label{tab:threshsens}
\small
\setlength{\tabcolsep}{4.5pt}
\begin{tabular}{lllll}
\toprule
Dataset & Paper's charted set (cells) & KS & Alternative sets (cells) & KS \\
\midrule
faults & 4 charted (20) & $0.117\pm0.008$ & 2 charted (5) & $0.548\pm0.006$ \\
heloc  & none (15)      & $0.158\pm0.014$ & 1 charted (10) & $0.132\pm0.008$ \\
lcld   & none (20)      & $0.133\pm0.007$ & 1 charted (5) & $0.134\pm0.007$ \\
url    & 1 charted (9)  & $0.144\pm0.021$ & none (16) & $0.150\pm0.011$ \\
news   & 2 charted (6)  & $0.297\pm0.007$ & none or 1 charted (19) & $0.292$ to $0.296$ \\
wids   & 1 charted (6)  & $0.133\pm0.008$ & 11 distinct sets (19) & $0.128$ to $0.152$ \\
\bottomrule
\end{tabular}
\end{table}

\paragraph{RQ4 in table form.} Table~\ref{tab:rq4} restates the RQ4 numbers of
Section~\ref{sec:rq4}.

\begin{table}[H]\centering
\caption{\textbf{RQ4: predicate placements on the high-$R$ orderings} (ordering
satisfaction, the per-ordering satisfaction fraction averaged over the dataset's orderings,
$n{=}10$, mean $\pm$ std). ``chance'' is a free generator with no constraint
mechanism. ``G-LTN-GAN'' places the predicate in the generator loss. ``D re-weight''
re-weights each sample's discriminator loss by its predicate satisfaction. ``D feature aug.''
appends the predicate value to the discriminator's input. FSG-LTN-GAN is function-symbol
grounding. No placement of the predicate moves the ordering more than $0.09$ above chance,
and the discriminator-feature placement falls well below it. Function-symbol grounding
satisfies the ordering for every sample.}
\label{tab:rq4}
\small
\begin{tabular}{lccccc}
\toprule
Dataset & chance & G-LTN-GAN & D re-weight & D feature aug. & \textbf{FSG-LTN-GAN} \\
\midrule
Alchemy & $0.49\pm0.03$ & $0.58\pm0.02$ & $0.51\pm0.03$ & $0.14\pm0.08$ & $\mathbf{1.000\pm0.000}$ \\
tmQM    & $0.53\pm0.08$ & $0.59\pm0.03$ & $0.55\pm0.08$ & $0.29\pm0.07$ & $\mathbf{1.000\pm0.000}$ \\
\bottomrule
\end{tabular}
\end{table}

\bodysection{Code and Data Availability}\label{app:code}  
All datasets are publicly available (Appendix~\ref{app:data}); code, experiment scripts, and the
prediction record of Appendix~\ref{app:extra} are available at
\href{https://github.com/nuuoe/FSG-LTN-GAN}{FSG-LTN-GAN}.

\end{document}